%% file: arxiv.tex
\documentclass[runningheads]{llncs}

\usepackage{eccv}

\usepackage{eccvabbrv}

\usepackage{graphicx}
\usepackage{booktabs}

\usepackage[accsupp]{axessibility}  %

\input{math_commands.tex}

\usepackage{anyfontsize}
\usepackage{amsmath,amsfonts,amssymb} 
\usepackage{bbm}
\usepackage{enumitem}
\usepackage{pifont}
\usepackage{caption}
\usepackage{dsfont}
\usepackage{tabu}
\usepackage{subcaption}
\usepackage{wrapfig}
\usepackage{multirow}
\usepackage{makecell}
\usepackage{hhline}
\usepackage{tabularx}
\usepackage[export]{adjustbox}

\newcommand{\bftab}{\fontseries{b}\selectfont}

\usepackage{mathtools}

\newcommand{\mytilde}{\raise.17ex\hbox{$\scriptstyle\mathtt{\sim}$}}

\usepackage[linesnumbered,ruled,vlined]{algorithm2e}
\newcommand{\nosemic}{\renewcommand{\@endalgocfline}{\relax}}%
\newcommand{\dosemic}{\renewcommand{\@endalgocfline}{\algocf@endline}}%
\let\oldnl\nl %
\newcommand{\nonl}{\renewcommand{\nl}{\let\nl\oldnl}}%

\usepackage[disable]{todonotes} %

\newcommand{\danica}[2][noinline]{\todo[color=violet!20,#1]{Danica: #2}}

\usepackage{tikz}

\DeclareCaptionFont{blue}{\color{blue}}

\definecolor{dark2green}{rgb}{0.1, 0.65, 0.3}
\definecolor{dark2orange}{rgb}{0.9, 0.4, 0.}
\definecolor{dark2purple}{rgb}{0.4, 0.4, 0.8}

\newcommand{\first}[1]{\textbf{\textcolor{dark2green}{#1}}}
\newcommand{\second}[1]{\textbf{\textcolor{dark2orange}{#1}}}
\newcommand{\third}[1]{\textbf{\textcolor{dark2purple}{#1}}}

\newcommand{\C}{\mathcal{C}}
\newcommand{\D}{\mathcal{D}}
\newcommand{\indic}{\mathds{1}}
\newcommand{\one}{\mathbf{1}}
\newcommand{\T}{\mathcal{T}}
\newcommand{\tp}{^{\mathsf{T}}}
\newcommand{\U}{\mathcal{U}}
\newcommand{\w}{\mathbf w}
\newcommand{\zero}{\mathbf{0}}

\DeclareMathOperator{\Pick}{Pick}
\newcommand{\train} {\mathit{train}}
\newcommand{\eval}  {\mathit{eval}}
\newcommand{\deploy}{\mathit{deploy}}

\DeclareRobustCommand\onedot{\futurelet\@let@token\@onedot}
\def\onedot{. } %
\def\eg{\emph{e.g}\onedot}

\def\etal{\emph{et al}\onedot}

\usepackage{hyperref}

\usepackage{orcidlink}

\usepackage[capitalize]{cleveref}

\usepackage{amsthm}
\usepackage{thmtools}

\begin{document}

\title{Exploring Active Learning in Meta-Learning: Enhancing Context Set Labeling} 

\titlerunning{Active Learning in Meta-Learning: Enhancing Context Set Labeling}

\author{Wonho Bae\inst{1} \and
Jing Wang\inst{1}\and
Danica J. Sutherland\inst{1,2}}

\authorrunning{W.~Bae et al.}
\institute{University of British Columbia, Vancouver \and
Alberta Machine Intelligence Institute (Amii) \\
\email{whbae@cs.ubc.ca}, \email{jing@ece.ubc.ca}, \email{dsuth@cs.ubc.ca}}

\maketitle

\begin{abstract}
  Most meta-learning methods assume that the (very small) context set used to establish a new task at test time is passively provided.
    In some settings, however, it is feasible to actively select which points to label; the potential gain from a careful choice is substantial, but the setting requires major differences from typical active learning setups.
    We clarify the ways in which active meta-learning can be used to label a context set, depending on which parts of the meta-learning process use active learning.
    Within this framework, we propose a natural algorithm based on fitting Gaussian mixtures for selecting which points to label; though simple, the algorithm also has theoretical motivation.
    The proposed algorithm outperforms state-of-the-art active learning methods when used with various meta-learning algorithms across several benchmark datasets.
  \keywords{Meta learning \and Active learning \and Low budget}
\end{abstract}

\section{Introduction}
\label{sec:intro}

Meta-learning has gained significant prominence as a substitute for traditional ``plain'' supervised learning tasks,
with the aim to adapt or generalize to new tasks given extremely limited data.
(Hospedales~\etal\cite{meta-survey} give a recent survey.)
There has been enormous success compared to learning ``from scratch'' on each new problem,
but could we do even better, with even less data?

One major way to improve data-efficiency in standard supervised learning settings
is to move to an \emph{active} learning paradigm,
where typically a model can request a small number of labels from a pool of unlabeled data;
these are collected, used to further train the model, and the process is repeated.
(Settles~\etal\cite{al2009settles} provides a classic overview, and Ren~\etal\cite{deep-al-survey} a more recent survey.)
Although each of these lines of research are well-developed,
their combination -- \textit{active meta-learning} -- has seen comparatively little research attention.
This intersection, however, is not only theoretically appealing but also has numerous practical applications.
For example, in medical imaging, there is often a large repository of labeled X-ray or MRI images. 
However, labeling these images requires manual annotations by radiologists or pathologists, which is time-consuming and costly. 
In manufacturing, there is often a large amount of sensor data generated during the production process. 
Annotating them to identify patterns indicating quality issues requires domain experts.

How can a meta-learner leverage an active learning setup to learn the best model possible, using only a few labels in its context sets?
We are aware of three previous attempts at active selection of context sets in meta-learning:
M\"uller~\etal\cite{fasl2022muller} and  Al-Shedivat~\etal\cite{hybrid_active_meta2021al} do so at meta-\emph{training} time for text and image classification,
while Boney~\etal\cite{active_fewshot_pn2017boney} do it at meta-\emph{test} time in semi-supervised few-shot image classification with ProtoNet~\cite{prototypical2017snell}.
``Active meta-learning'' thus means very different things in their procedures;
these approaches are also entirely different from work on active selection of \emph{tasks} during meta-training 
\cite{kaddour2020probabilistic,bamld2021nikoloska,active_task2022kumar}.
Our first contribution is therefore to clarify the different ways in which active learning can be applied to meta-learning, for differing purposes.%
\footnote{Note that work on meta-learning an active selection criterion for higher-label-budget problems -- \eg \cite{konyushkova2017learning,fang-etal-2017-learning} -- is essentially unrelated.}

We then confirm in extensive experiments that
no active learning method for context set selection seems to significantly help with final predictor quality at meta-training time -- aligning with previous observations by Setlur~\etal\cite{support2020setlur} and Ni~\etal\cite{aug_meta2021ni} --
but that active learning \emph{can} substantially help at meta-test time.
In particular, we propose a natural algorithm
based on fitting a Gaussian mixture model to the unlabeled data, using meta-learned feature representations;
though the approach is simple, we also give theoretical motivation.
We show that our proposed selection algorithm 
works reliably, and often substantially outperforms competitor methods
across many different meta-learning and few-shot learning tasks,
across a variety of benchmark datasets
and meta-learning algorithms.
Our contributions are summarized as follows.
\begin{itemize}
    \item We explore the concept of ``active meta-learning,'' pointing out that active selection can occur in several places (\cref{subsec:background_ml}) and highlighting challenges in traditional meta-learning setups regarding sample stratification (\cref{subsec:active_meta}).
    \item We identify that existing active learning algorithms, particularly low-budget active learning methods, perform poorly in active meta-learning
    (\cref{sec:experiments}),
    even if meta-learning is in the low-budget regime.
    \item We propose a method based on Gaussian mixture model using meta-learning specific features (\cref{subsec:feature-choice}), proven to yield a Bayes-optimal classifier under certain assumptions (or an efficient set cover in a general setting in \cref{subsec:gmm}).
    \item Our experiments show that the simple Gaussian mixture method consistently outperforms more complex active learning methods across various few-shot image classification (\cref{subsec:few_shot_image_cls,subsec:hybrid}), cross-domain classification (\cref{subsec:few_shot_cross_domain}), and meta-learning regression tasks (\cref{subsec:regression}), irrespective of the type of meta-learning algorithms employed.
\end{itemize}

\section{Meta-Learning: Background and Where to be ``Active''}
\label{subsec:background_ml}

We aim to learn a learning algorithm $f_\theta$,
a function which, given a dataset $\C$ consisting of pairs $(x, y) \in \mathcal X \times \mathcal Y$, returns $g := f_\theta(\C)$.
The function $g : \mathcal X \to \hat{\mathcal Y}$
is a classifier, regressor, or so on.
We evaluate the quality of $g$ using a loss function
$\ell : \hat{\mathcal Y} \times \mathcal Y \to \mathbb R$,
e.g.\ the cross-entropy or square loss:
\[
    \text{\emph{Empirical risk} of $g$ on $\T$:}\;
    \mathcal R_\ell(g, \T)
    =
    \frac{1}{\lvert \T \rvert} \sum_{(x, y) \in \T}
    \ell\left( g(x), y \right)
.\]
To find the $\theta$ which gives the best $g$s,
we assume we have access to distributions $\mathcal P^\train, \mathcal P^\eval$ over tasks
$\D \subseteq \mathcal X \times \mathcal Y$.
For each task, we will run $f_\theta$ on a \emph{context set}
$\C$, then evaluate the quality of the learned predictor on a disjoint \emph{target set} $\T$.
We call the distribution over possible $(\C, \T)$ pairs $\Pick_\theta(\D)$.\footnote{If we pick points by some deterministic process, $\Pick_\theta(\D)$ is a point mass.}
For instance, the default choice in passive meta-learning chooses, say, five random points per class for $\C$
and assigns the rest to $\T$,
ignoring $\theta$ and $x$.
Our aim is then,
\begin{gather}
    \intertext{\qquad Meta-training: find $\hat\theta$ using}
    \hat\theta \approx \argmin_\theta
    \E_{\D \sim \mathcal P^\train} \Bigg[ \E_{(\C, \T) \sim \Pick_\theta^\train(\D)} \Big[
    \mathcal R_{\ell^\train}\big(
        f_\theta(\C), \T
    \big) \Big] \Bigg]
\label{eq:meta_train}
.\end{gather}
Many algorithms have been proposed for meta-training (overviewed in \cref{subsec:meta_learning}).

To compare models based on $\mathcal P^\eval$, we might evaluate with a different loss.
For instance, it would be typical to use the 0-1 loss (corresponding to accuracy) for classification problems, despite training with cross-entropy.
\begin{gather}
    \text{Meta-testing: eval. $f_{\hat\theta}$ using}\;
    \E_{\widetilde{\D} \sim \mathcal P^\eval} \Bigg[ \E_{(\widetilde{\C}, \widetilde{\T}) \sim \Pick_{\hat\theta}^\eval(\widetilde{\D})} \Big[
    \mathcal R_{\ell^{\eval}}\big(
        f_{\hat\theta}(\widetilde{\C}), \widetilde{\T}
    \big) \Big] \Bigg]
\label{eq:meta_test}
.\end{gather}
Finally, in practice,
we might want to use a different selection scheme at deployment time.
For instance, in passive meta-learning,
one would typically use all available labeled data for context, not a random subset. 
Given a task $\breve{\D}$,
\begin{gather}
    \text{Deployment: 
    find a context set via $(\breve{\C}, \_) \sim \Pick_{\hat\theta}^\deploy(\breve{\D})$
    and use $f_\theta(\breve{\C})$}
\label{eq:meta_deploy}
.\end{gather}

\subsection{Active Selection of Context in Meta Learning}
\label{subsec:active_meta}

There are several places where active learning can be applied during meta-learning.
In the meta-training phase \eqref{eq:meta_train},
we could actively choose tasks $\D$,
and/or have $\Pick_\theta^\train$ actively select points for $\C$ and/or $\T$.
At meta-testing time \eqref{eq:meta_test},
we could have $\Pick_\theta^\eval$ actively select points for $\widetilde{\C}$ and/or $\widetilde{\T}$;
we might also actively choose $\widetilde{\D}$ to use labels efficiently, similarly to active surveying \cite{garnett2012bayesian}.
At deployment time \eqref{eq:meta_deploy},
$\Pick_\theta^\deploy$ might actively choose a context set $\breve\C$ to label.

Actively selecting $\D$, $\widetilde{\D}$, $\T$, and/or $\widetilde{\T}$ is interesting to minimize the label burden (or, possibly, computational cost) of meta-training
\cite{kaddour2020probabilistic,bamld2021nikoloska,active_task2022kumar}.
We assume here, however,
that $\mathcal P^\train$ and $\mathcal P^\eval$ are based on already-labeled datasets.

Instead, we are primarily concerned with the labeling burden at deployment time, and so our final goal is to actively select $\breve{\C}$ with $\Pick_\theta^\deploy$ to find the best predictor.
To evaluate how well we should expect our algorithms to perform at this task, we choose $\Pick_\theta^\eval = \Pick_\theta^\deploy$; thus, we actively select $\widetilde\C$.

Should we expect this to help?
Efficient approaches for data selection in meta-learning have not yet received much research attention.
Setlur~\etal\cite{support2020setlur} suggest that context set diversity is empirically not particularly helpful for meta-learning,
and Ni~\etal\cite{aug_meta2021ni} show that data augmentation on context sets is not very useful either.
Pezeshkpour~\etal\cite{active_fewshot2020pezeshkpour} further provide some evidence using label information that there is not much room to improve few-shot classification with active learning. 
Agarwal~\etal\cite{sensitivity2021agarwal}, however, argue against these previous findings, showing that adversarially selected context sets, at both training and test time, significantly change the classification performance.
Their approach is not applicable in practice since it requires full label information,
but may suggest there is room to improve meta-learning algorithms with better context sets.

Muller~\etal\cite{fasl2022muller} and Al-Shedivat~\etal\cite{hybrid_active_meta2021al} compare traditional active learning algorithms for few-shot text and image classification at training time, i.e.\ active $\Pick_\theta^\train$, passive $\Pick_\theta^\eval$.
Boney~\etal\cite{active_fewshot_pn2017boney} instead compare active learning algorithms inside $\Pick_\theta^\eval$, specifically when $f_\theta$ is a ProtoNet, %
with passive $\Pick_\theta^\train$.
(They do this in semi-supervised few-shot image classification; more discussion on the relationship of active meta-learning to active semi-supervised few-shot learning is provided in \cref{app:sec:semi}.)
Active $\Pick_\theta^\train$ and $\Pick_\theta^\eval$ are both feasible settings,
but as argued above, if we are concerned with performance of our deployed predictor
we should use an active $\Pick_\theta^\eval = \Pick_\theta^\deploy$.
One can choose $\Pick_\theta^\train$ to be active or not, depending on which learns better predictors;
we show in \cref{app:sec:train_time} that active $\Pick_\theta^\train$ does not seem to help.

\begin{figure}[t!]
    \centering
    \includegraphics[width=0.9\textwidth]{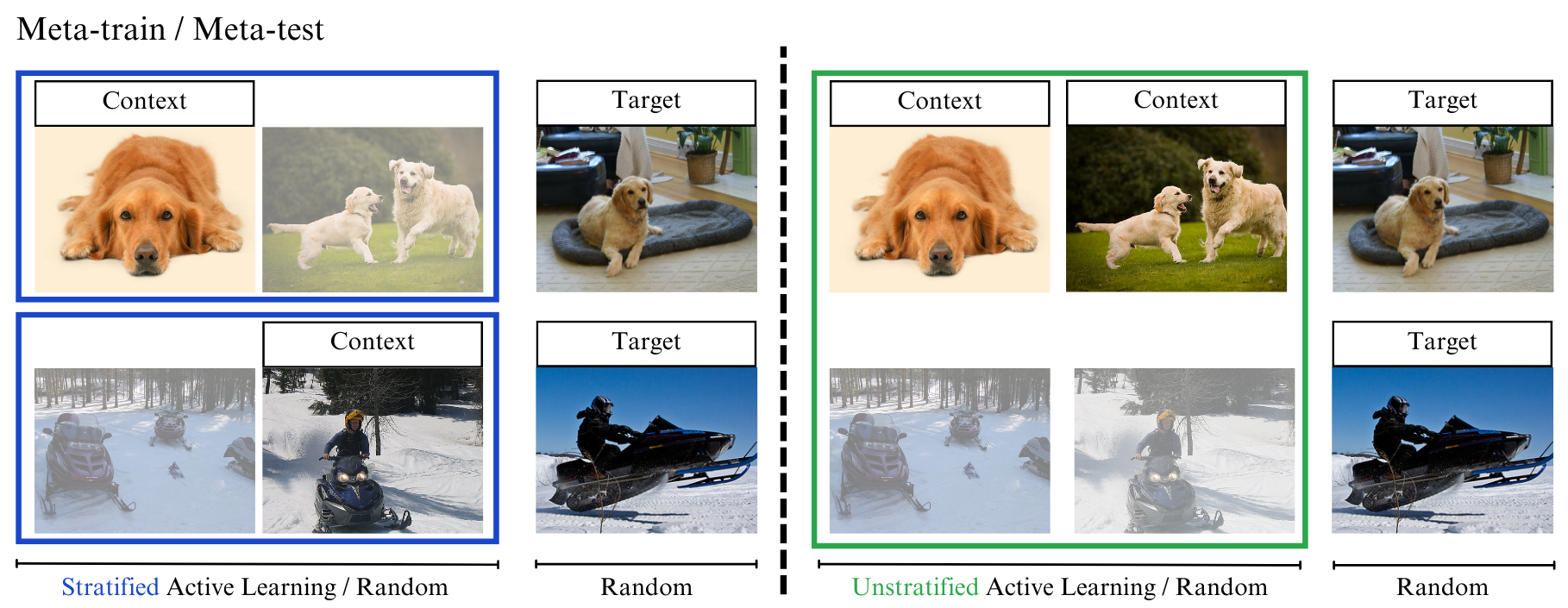}
    \caption{Meta-training process. $\Pick_\theta$ can be stratified or unstratified, active or passive.}
    \label{fig:active_meta}
\end{figure}

\subsubsection{Stratification}
In passive few-shot classification,
the $\Pick$ functions typically choose context points according to a \emph{stratified} sample:
for one-shot classification, $\C$ contains exactly one point per class.
This is because, if we take a uniform random sample of size $N$ for an $N$-way classification problem,
$\C$ is unlikely to contain all the classes,
making classification very difficult.
Assuming ``nature'' gives a stratified uniform sample, as in nearly all work on few-shot classification, also seems reasonable.

In pool-based active settings, however,
it is highly unreasonable to assume that $\Pick_\theta^\deploy$ can be stratified (as illustrated on the left side of \cref{fig:active_meta}):
to do so, we would need to know the label of every point in $\breve\D$,
in which case we should simply use all those labels.
As we would like $\Pick_\theta^\eval = \Pick_\theta^\deploy$, eval-time stratification is then not particularly reasonable;
even so, we do report such results per the standards of meta-learning.
When $\Pick_\theta^\deploy$ is unstratified (as in the right side of \cref{fig:active_meta}),
it is particularly important for the selection criterion to find samples from each class.

Train-time stratification with unstratified evaluation does not leak data labels, and is
plausible when $\mathcal P^\train$ and $\mathcal P^\eval$ are fully labeled.
Since this approach trains $f_\theta$ in an ``easy'' setting and evaluates it in a ``hard'' one, however,
we will see it tends to slightly underperform the fully-unstratified default.
Regression tasks are not typically stratified;
we do not stratify for regression experiments.

\subsection{Related Work: Meta-Learning algorithms}
\label{subsec:meta_learning}
Meta-learning algorithms can be divided into several categories;
all will be applicable for our active learning strategies,
and we evaluate with at least one representative algorithm per category.

\textbf{Metric-based methods}
learn a representation space encoding a ``good'' similarity, where simple classifiers work well \cite{matching2016vinyals,tadam2018oreshkin}.
ProtoNet~\cite{prototypical2017snell} finds features so that points from each class are close to the prototype feature of the class.

\textbf{Optimization-based methods}
use $f_\theta$ that incorporate optimization, e.g.\ gradient descent
as in MAML~\cite{maml2017finn,maml++2018antoniou}, which seeks parameters $\theta$ 
such that gradient descent quickly finds a useful model on a new task.
ANIL~\cite{anil2019raghu} freezes most of the network and only updates the last layer,
while R2D2~\cite{r2d2} and MetaOptNet~\cite{metaopt2019lee} replace the last layer with a convex problem whose solution can be differentiated;
these approaches can improve both performance and speed.

\textbf{Model-based methods}
learn a model that explicitly adapts to new tasks,
typically by modeling the distribution of $y$ from $\T$ given its $x$ values and $\C$.
The most prominent family of methods is
Neural Processes (NPs)~\cite{np2018garnelo,npf2020dubois},
which encode a context set and estimate task-specific distribution parameters. 
Conditional NPs can have issues with underfitting
\cite{cnp2018garnelo2018,np2018garnelo},
but
AttentiveNPs~\cite{attnnp2019kim}
and ConvNPs~\cite{convnp2019gordon}
are more robust.
They are more commonly used for regression.

\textbf{Pre-training methods},
such as SimpleShot \cite{simpleshot2019wang}
and Baseline++ \cite{baseline++2019chen},
are based on
repeated demonstrations~\cite{qda2021zhang,unraveling2020goldblum} that simply pre-training a multi-class model can surpass the performance of commonly used meta-learners.

\section{Active Learning}
\label{sec:active_learning}

In pool-based active learning, a model requests labels for the most ``informative'' data points from a pool of unlabeled data. 
The key question is how to estimate which data points will be informative.

\subsection{Related Work: Existing Active Learning Methods}

\subsubsection{Uncertainty-based methods}
Simple but effective uncertainty-based methods such as maximum entropy~\cite{entropy2014wang}, least confident~\cite{al2009settles}, and margin sampling~\cite{margin2001scheffer} are widely used for active learning.
Since they only consider current models' uncertainty, active learning strategies that consider expected changes in model parameters~\cite{egl2007settles,badge2019ash} and model outputs~\cite{eer_gf2003zhu,eer_mi2007guo,eer2001roy,emoc_reg2018kading,emoc2014frey,emoc2016kading,bemps2021tan,mohamadi:active-ntk} have been also been proposed. 
However, recent analyses have empirically demonstrated that at least in certain experimental settings, most active learning methods are not significantly different from one another~\cite{practice_al2021lang}, and may not even improve over random selection~\cite{reproduce_al2022munjal}.
We consider the following methods from this category:
\begin{description}[itemsep=2pt,parsep=2pt,leftmargin=!,labelindent=1em,labelwidth=5em,itemindent=4em,labelsep=1ex,]
    \item[Random] Uniformly randomly samples a context set from unlabeled set $\mathcal{U}$.
    \item[Entropy]
    Add a point to the context set based on $x^* = \argmax_{x \in \mathcal{U}} H(\hat y(x) \mid x)$, where $H(\cdot)$ is Shannon entropy~\cite{entropy2014wang}.
    Other than in \cref{app:sec:sequential},
    we apply this in ``batch mode,''\danica{update discussion}{} i.e.\ we do not observe points one-by-one but rather choose the $\lvert\C\rvert$ points with the highest ``initial'' entropy.\footnote{%
        Traditional active learning methods would generally retrain between each step,
        requiring a back-and-forth labeling process not needed by the methods discussed shortly.
        In modern deep learning settings, this is almost never done due to the expense of retraining;
        ``batch-mode'' entropy is still excellent in those settings \cite{practice_al2021lang,mohamadi:active-ntk}.
        \Cref{app:sec:sequential} explores more frequent retraining;
        the takeaway results are overall similar to the rest of our experiments.}
    \item[Margin] Add a point to $\C$ based on $x^* = \argmin_{x \in \mathcal{U}} p_1(y|x) - p_2(y|x)$, where $p_1$ and $p_2$ denote the first and second highest predicted probabilities, respectively~\cite{margin2001scheffer}. 
        We also run this method in ``batch mode.''
\end{description}

Although Entropy and Margin are very simple and fast to evaluate, no uncertainty-based method seems to substantially outperform them on typical active image classification tasks (see \eg \cite{mohamadi:active-ntk}),
and we will see that other methods are unlikely to be competitive in low-budget regimes.

\subsubsection{Low-budget active learning}
The limitations of typical active learning approaches may especially apply in very-low-budget cases,
such as those considered in few-shot classification and meta-learning.
In particular, when the ``current'' model is quite bad, using it to choose points might be counterproductive.
In the one-shot case especially, standard active learning methods simply do not apply.

Recently, several papers have have proposed novel active learning algorithms for these settings; none of these papers focused on meta-learning,
but should be broadly applicable since meta-learning is also a low-budget setting.
Rather than picking \eg the points about which a model is least certain,
these papers propose to label the ``most representative'' data points independently of a ``current'' model.
\begin{description}[itemsep=2pt,parsep=2pt,leftmargin=!,labelindent=1em,labelwidth=5em,itemindent=4em,labelsep=1ex,]
    \item[DPP] Determinantal Point Processes (DPPs) 
    query diverse samples, by selecting a subset that maximizes the determinant of a similarity matrix~\cite{batch_dpp2019biyik}.
    \item[Typiclust]
    Run $k$-means on the unlabeled data points, where $k = \lvert \C \rvert$ is the annotation budget.
    Select one data point per cluster such that the distance between a data point and its $k'$ nearest neighbors is minimized: $\argmin_{x \in \mathcal{U}} \sum_{x' \in \operatorname{NN}_{k'}(x)} \lVert x - x' \rVert_2 $~\cite{typiclust2022highlow}.
    \item[Coreset] Select a subset of the unlabeled set $\mathcal{U}$ to approximately minimize the distance from unlabeled data points to their nearest labeled point~\cite{coreset2017sener}. %
    \item[ProbCover] Select data points that roughly maximize the number of unlabeled points within a distance of $\delta$ from any labeled point, where $\delta$ is chosen according to a ``purity'' heuristic~\cite{probcover2022yehuda}; see \cref{app:sec:prob_cover} for more details.
\end{description}

\subsection{Features for Representative-Selection Methods}
\label{subsec:feature-choice}
Notions of the ``most representative'' data points
are highly dependent on a reasonable metric of data similarity.
Prior methods operated either on raw data
-- typically a poor choice for complex datasets like natural images --
or, in semi-supervised settings as in ProbCover %
and Typiclust, %
on SimCLR~\cite{simclr2020chen} features learned on the unlabeled data.

In metric-based meta-learning,
we propose to instead use the current meta-learned representation;
choosing points representative for the features we will use downstream is the natural choice.
In MAML, the most natural equivalent might be features from the empirical neural tangent kernel (NTK)~\cite{linntk2019lee} of the current initialization network;
this approximates what will happen when the network is trained on $\C$,\footnote{
    Theoretical results about the NTK technically depend on a random initialization, which is not the case here.
    Mohamadi~\etal\cite{mohamadi:active-ntk}
    provide some assurance in that if the initialization were obtained by gradient descent on some dataset, the results would still hold, but MAML finds initial parameters differently.
    }
and so is perhaps the best simple understanding of ``how this network views the data.''
Even empirical NTKs are often expensive to evaluate, however,
and we thus propose to instead use
features from the penultimate layer of the initialization neural net $f_\theta(\{\})$,
corresponding to the NTK of a model that only retrains its last layer (as in ANIL, R2D2, and MetaOptNet).
We also use the penultimate-layer reperesentations of $f_\theta(\{\})$ for NP-based meta-learning.

Experiments in \cref{app:sec:self_sup} show that this proposal outperforms off-the-shelf self-supervised features like SimCLR.

\subsection{Gaussian Mixture Selection for Low-Budget Active Learning}
\label{subsec:gmm}

We propose the following very simple algorithm for low-budget active learning:
fit a mixture of $k$ Gaussians to the unlabeled data features,
where $k$ is the label budget,
using EM with a $k$-means initialization.
We use a shared diagonal covariance matrix (more details about EM are provided in \cref{app:sec:gmm}).
Once a mixture is fit,
we select the highest-density point from each component:
\begin{align}
    x^* = \argmin_{x \in \mathcal U} (x - \mu_j)\tp \Sigma^{-1} (x - \mu_j) 
\text{ for each } j \in [k].
\end{align}
The proposed method is summarized in \cref{algo:ours}.
For metric-based meta-learning,
the motivation of this algorithm is clear:
we want labeled points that approximately ``cover'' the data points.
Our notion of a ``cover'' is somewhat different from that of Coreset \cite{coreset2017sener} or ProbCover \cite{probcover2022yehuda};
we avoid ProbCover's need for a fixed radius,
which we show can lead to poor choices (see \cref{app:sec:prob_cover}),
and are more concerned with ``average'' covering (and hence perhaps less sensitive to outliers) than Coreset.
The quality of selected data points from those methods are compared for a few metrics in \cref{app:fig:al_quality}.
On ANIL and MetaOptNet:
since $\lvert \C \rvert$ is at most, say, 50 (in $10$-way $5$-shot)
and the feature dimension is typically at least 100,
ANIL becomes approximately the same multi-class max-margin separator obtained by (unregularized) MetaOptNet.\footnote{%
    For reasonable distributions and networks,
    $\C$ is almost surely linear separable;
    thus ANIL, which is gradient descent for logistic regression, will converge to the multi-class max-margin separator \cite{soudry:implicit}.}
Intuitively, as $\lvert\C\rvert$ grows, the means of an isotropic Gaussian mixture converge to roughly a covering set for the dataset $\U$,
and the max-margin separator of a set cover for $\U$ will be similar to the max-margin separator for all of the data.
Even in various cases when $\lvert\C\rvert \ll \lvert\U\rvert$,
choosing the means yields a max-margin separator that generalizes well.

\Cref{fig:max_margin} illustrates that,
if class-conditional data distributions are isotropic Gaussians with the same covariance matrices,
labeling the cluster centers can be far preferable to
labeling a random point from each cluster.
This is backed up by the following theoretical results,
which are all proved in \cref{app:sec:max-margin}.

\setlength{\textfloatsep}{0pt}
\begin{algorithm}[t!]
    \newlength{\commentWidth}
    \setlength{\commentWidth}{7cm}
    \newcommand{\atcp}[1]{\tcp*[r]{\makebox[\commentWidth]{#1\hfill}}}
    \SetKwInput{KwInput}{Input}
    \SetKwInput{MetaTrain}{In meta-train}
    \SetKwInput{MetaTest}{In meta-test}
    \KwInput{Selection distribution $\Pick_\theta^{\{\train,\eval\}}$, a learning algorithm $f_\theta$, empirical risk $\mathcal R_\ell$, and the size of context sets $k$ }
   Find $\hat\theta$ using \cref{eq:meta_train} where $\Pick_\theta^\train$ may be stratified \DontPrintSemicolon \tcp*{\textcolor{blue}{Meta-train}}
    \While(\tcp*[f]{\textcolor{blue}{Meta-test}} ){\emph{task for evaluation exists}}{
    $\tilde{\D} \sim P^\eval$ and sample $\tilde{\T}$ from $\tilde{\D}$\\
    Fit GMM: $\{(\hat{\pi}_j, \hat{\mu}_j, \hat{\Sigma}_j) \}_{j=1}^k$ using \cref{app:eq:gmm_obj}--\labelcref{app:eq:m_step} in \cref{app:sec:gmm} \\
    Select $\{x_j^*\}_{j=1}^k$ such that $\forall j \in [k]$, $x_j^* = \argmin_{x \in \mathcal{X}} (x - \hat{\mu}_j)^T\hat{\Sigma}_j(x - \hat{\mu}_j)$\\
    Annotate $\{x_j^*\}_{j=1}^k$ to create $\tilde{\C}$, and evaluate $f_{\hat{\theta}}$ using $R_{\ell}\big(f_{\hat{\theta}(\tilde{\C})}, \tilde{\T} \big)$
    }
    \caption{GMM-based Active Meta-learning}
    \label{algo:ours}
\end{algorithm}

\begin{restatable}{proposition}{maxmargin}
    \label{lemma:orthonormal-maxmarg}
    Suppose that $\{ x_i \}_{i=1}^N$ are orthonormal.
    Then, the solution to \eqref{eq:multiclass-svm}
    with the dataset $\{ (x_y, y) \}_{y=1}^N$
    is given by $w_y = x_y - \frac1N \sum_{i=1}^N x_i$,
    and hence
    \begin{align}
        \text{for any } x,\quad
        \argmax_y w_y\tp x = \argmin_y \lVert x - x_y \rVert
    .\end{align}
\end{restatable}

\Cref{lemma:orthonormal-maxmarg} says with orthonormal data points, a $N$-class support vector machine (one form of max-margin separators) defined in \cref{eq:multiclass-svm} becomes a nearest-neighbor classifier.
While these assumptions will not exactly hold in practice,
for high-dimensional normalized features,
it is reasonable to expect our selected data points to be \emph{almost} orthonormal.
In combination with  \cref{lemma:iso-lda} (in \cref{app:sec:max-margin}),
this leads to the following optimality result.

\begin{restatable}{corollary}{combined}
\label{prop:max-margin}
    Suppose $Y \sim \mathrm{Uniform}([N])$,
    and $X \mid (Y = y) \sim \mathcal N(\mu_y, \sigma^2 I)$,
    where the $\mu_i$ are orthonormal.
    Then the max-margin separator \eqref{eq:multiclass-svm}
    on $\{ (\mu_i, i) \}_{i=1}^N$
    is Bayes-optimal for $Y \mid (X = x)$.
\end{restatable}

For more general settings, we argue that GMM is still a good method based on being an efficient set cover, as shown in \cref{app:fig:max_margin_aniso} in \cref{app:sec:max-margin}.

\subsubsection{Very-low-budget regime}
Active learning based on Gaussian mixtures is not new in and of itself.
Closely-related methods such as $k$-means, $k$-means$^{++}$ or $k$-medoids have been employed either as standalone selection algorithms~\cite{kmedoids2016aghaee,kmedian2012voevodski} or in combination with uncertainty-based methods~\cite{weigted_logistic2004nguyen,dual2007donmez,badge2019ash,typiclust2022highlow}.
Some recent work~\cite{active_fewshot_pn2017boney}  including DPP~\cite{batch_dpp2019biyik} and Coreset~\cite{coreset2017sener}
show significant improvements over $k$-means baselines.
These trends, however, do not seem to hold true in the very-low-budget scenarios typically encountered in meta-learning.
As shown in \cref{fig:img_cls}, GMM matches or outperforms other low-budget methods with very small numbers of labels for standard image classification tasks, which has not been known in the community.
The following section shows that GMM provides substantial improvements in meta-learning.

\begin{figure*}[t!]
    \centering
    \begin{subfigure}[b]{0.45\textwidth} %
        \includegraphics[width=\textwidth]{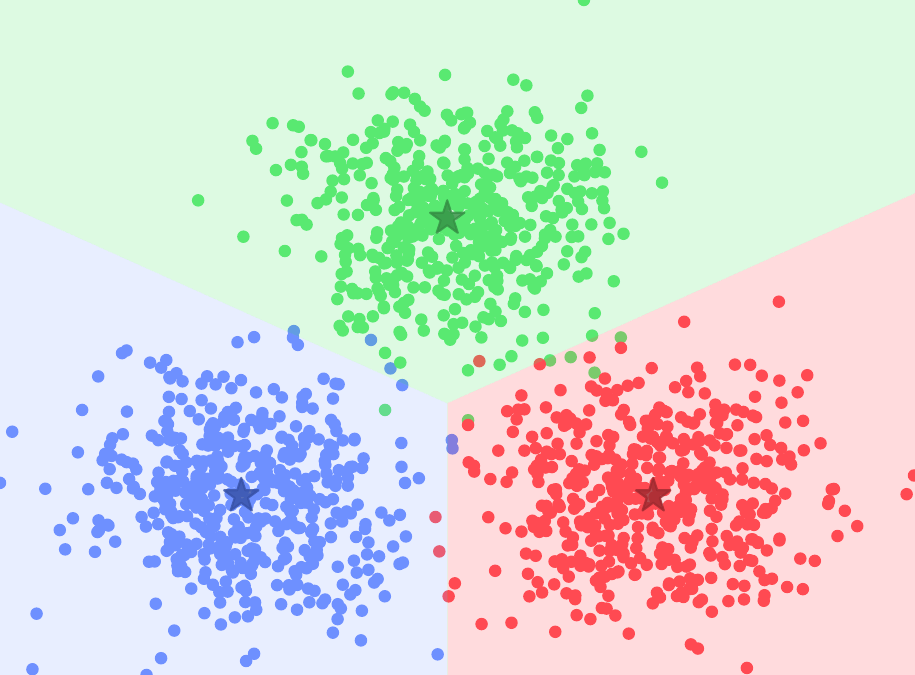}
        \caption{Trained on cluster centers}
        \label{app:fig:max_margin_centers}
    \end{subfigure}
    \hfill
    \begin{subfigure}[b]{0.45\textwidth} %
        \includegraphics[width=\textwidth]{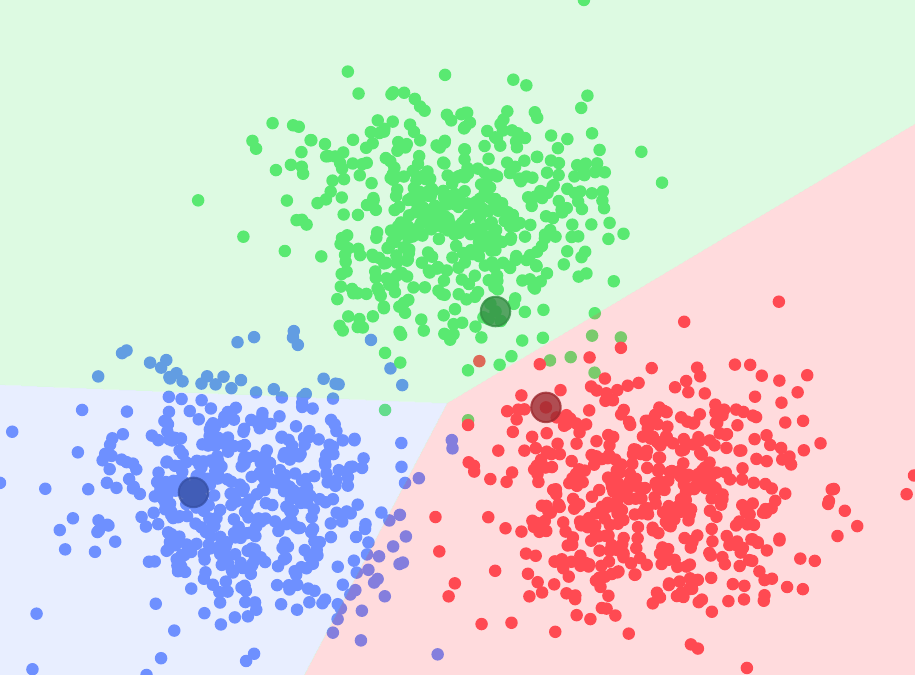}
        \caption{Trained on random points}
        \label{app:fig:max_margin_randoms}
    \end{subfigure}
    \caption{Decision boundaries using a multi-class SVM \eqref{eq:multiclass-svm} trained on a one-shot dataset containing (a) cluster centers (stars) and (b) randomly selected points (circles).}
    \label{fig:max_margin}
    \vspace{3mm}
\end{figure*}

\section{Active Meta-Learning Experiments} \label{sec:experiments}
We now compare various active learning methods for variants of active meta learning as defined in \cref{subsec:active_meta}, both for classification tasks (in \cref{subsec:few_shot_image_cls,subsec:few_shot_cross_domain,subsec:hybrid}) and regression (in \cref{subsec:regression}).

\subsection{Few-shot Image Classification}
\label{subsec:few_shot_image_cls}
We use four popular few-shot image classification benchmark datasets.
\textbf{MiniImageNet}~\cite{matching2016vinyals,metaLSTM2017ravi} consists of $60\,000$ images
with $64$ training classes, $16$ validation, and $20$ test. 
\textbf{TieredImageNet}~\cite{semi_fewshot2018ren} 
consists of $20$ training super-classes, $6$ validation, and $8$ test; each contains $10$ to $30$ sub-classes.
\textbf{FC100}~\cite{tadam2018oreshkin}
consists of $60$ training classes, $20$ validation, and $20$ test. 
\textbf{CUB}~\cite{cub2011wah,MACO2018hilliard} consists of $200$ classes of bird images,
with $140$ training classes, $30$ validation, and $30$ test.
We validate if our active learning method works across various types of meta-learning methods.
We run\footnote{%
    We reproduce %
    ProtoNet, MAML, and ANIL using Learn2Learn~\cite{learn2learn2020Arnold}; for MetaOptNet, Baseline++, and SimpleShot, we use repositories provided by the authors.
}
metric-based:
ProtoNet~\cite{prototypical2017snell}, optimization-based: MAML~\cite{maml2017finn}, ANIL~\cite{anil2019raghu}, and MetaOpt\cite{metaopt2019lee}, as well as pre-training-based: Baseline++~\cite{baseline++2019chen} and SimpleShot~\cite{simpleshot2019wang}.%
\footnote{%
    We do not run a model-based method on this case, though we will in \cref{subsec:regression};
    most variants do not work well conditioning on images.}
We vary the backbone to demonstrate robustness: for instance, we use 4 convolutional blocks for MAML and ProtoNet, and ResNet10~\cite{Resnet2015He} for Baseline++.
As typical in few-shot classification, we report means and 95\% confidence intervals for test accuracy
on 600 meta-test samples.

We use the meta-learner's features as proposed in \cref{subsec:feature-choice} for all methods;
experiments in \cref{app:sec:self_sup} confirm that they outperform contrastive learning of features on the meta-training set.
Additionally, in the main body we only present results where $\Pick_\theta^\train$ is random;
\cref{app:sec:train_time} demonstrates that, in our setup, active learning at train time is actually mildly \emph{harmful} to overall performance, aligning with observations by Ni~\etal\cite{aug_meta2021ni} and Setlur~\etal\cite{support2020setlur}. 

For \textbf{metric-based} methods,
\cref{tbl:proto_fc100} shows results for
ProtoNet %
on FC100.
The simple GMM method significantly outperforms the other active learning methods on all problems considered here.
As reported \cite{typiclust2022highlow,probcover2022yehuda},
uncertainty-based methods are significantly worse than random selection in this low-budget regime.

For \textbf{optimization-based}, \cref{tbl:maml_mini} shows results with MAML %
on MiniImageNet.
GMM again outperforms the other methods in most cases.
The performance of ProbCover is sometimes much lower than other methods due to its radius parameter, which is very difficult to tune, with the best choice changing dramatically depending on the sub-task although \cite{probcover2022yehuda} propose to fix this parameter per dataset (see \cref{app:sec:prob_cover} for more).
Additional results for ANIL %
on TieredImageNet and MetaOptNet %
on FC100 are provided in \cref{app:sec:additional_exp}.

For \textbf{pre-training-based} methods,
we compare active learning strategies with Baseline++ %
on the CUB dataset in \cref{tbl:baseline_cub},
seeing that the proposed method is again usually by far the best,
though in one five-shot case it essentially ties with DPP.
As these methods do not follow the meta-training process in \eqref{eq:meta_train},
train-time stratification is not applicable.
\Cref{app:sec:additional_exp} shows results for SimpleShot. %

\newcolumntype{D}{>{\centering\arraybackslash}p{5.9em}}
\begin{table*}[t!]
\centering
\fontsize{6.75}{9.75}\selectfont
\begin{tabu}{c|D|D|D|D|D|D}
\hline
\multirow{2}{*}{$\Pick_\theta^\eval$} & \multicolumn{3}{c|}{1-Shot} & \multicolumn{3}{c}{5-Shot}  \\
\cline{2-7}
 & Fully strat. & Train strat. & Unstrat. & Fully strat. & Train strat. & Unstrat. \\
\hline
\hline
Random & $ 36.73 \pm 0.18 $ & \third{$31.27 \pm 0.21$} & $31.40 \pm 0.41$ & \third{$47.98 \pm 0.18$} & $42.83 \pm 0.20$ & $44.00 \pm 0.21$ \\
Entropy & $33.67 \pm 0.16$ & $29.82 \pm 0.20$ & $30.01 \pm 0.20$ & $44.64 \pm 0.17$ & $38.39 \pm 0.22$ & $38.36 \pm 0.25$ \\
Margin & $34.28 \pm 0.18$ & $29.74 \pm 0.20$ & $28.99 \pm 0.20$ & $45.31 \pm 0.17$ & $39.65 \pm 0.21$ & $38.13 \pm 0.24$ \\
DPP& $36.20 \pm 0.18$ & $31.34 \pm 0.20$ & $31.09 \pm 0.20$ & $47.53 \pm 0.17$ & \third{$43.69 \pm 0.20$} & \third{$44.19 \pm 0.20$} \\
Coreset & $35.79 \pm 0.17$ & $30.31 \pm 0.20$ & \third{$31.57 \pm 0.18$} & $43.08 \pm 0.40$ & $41.56 \pm 0.20$ & $41.79 \pm 0.22$ \\
Typiclust & \third{$46.01 \pm 0.16$} & $30.96 \pm 0.19$ & $30.61 \pm 0.21$ & $47.54 \pm 0.17$ & $43.61 \pm 0.18$ & $44.03 \pm 0.21$ \\
ProbCover & \second{$48.66 \pm 0.16$} & \second{$32.86 \pm 0.22$} & \second{$33.58 \pm 0.19$} & \second{$51.11 \pm 0.17$} & \second{$44.20 \pm 0.23$} & \second{$44.40 \pm 0.24$} \\
GMM (Ours)& \first{$50.22 \pm 0.18$} & \first{$34.23 \pm 0.23$} & \first{$35.03 \pm 0.23$} & \first{$54.76 \pm 0.17$} & \first{$46.30 \pm 0.21$} & \first{$47.03 \pm 0.20$} \\
\Xhline{2\arrayrulewidth}
\end{tabu}
\vspace{2mm}
\caption{5-Way K-Shot on FC100 with ProtoNet, with $\Pick_\theta^\train$ random. The \first{first}, \second{second}, \third{third} best results for each setting are marked in this and all other tables.}
\label{tbl:proto_fc100}
\vspace{-1mm}
\end{table*}

\begin{table*}[t!]
\centering
\fontsize{6.75}{9.75}\selectfont
\begin{tabu}{c|D|D|D|D|D|D}
\hline
\multirow{2}{*}{$\Pick_\theta^\eval$} & \multicolumn{3}{c|}{1-Shot} & \multicolumn{3}{c}{5-Shot}  \\
\cline{2-7}
 & Fully strat. & Train strat. & Unstrat. & Fully strat. & Train strat. & Unstrat. \\
\hline
\hline
Random & $47.93 \pm 0.20 $ & $28.16 \pm 0.17$  & $34.85 \pm 0.19$  & \third{$64.16 \pm 0.18$} & \third{$53.54 \pm 0.20$} & \second{$58.84 \pm 0.20$} \\
Entropy & $48.16 \pm 0.20$ & $25.56 \pm 0.14$  & $30.44 \pm 0.17$  & $61.22 \pm 0.20$ & $34.36 \pm 0.23$ & $39.57 \pm 0.26$ \\
Margin & $48.31 \pm 0.20$ & $28.32 \pm 0.16$  & $30.83 \pm 0.17$  & $63.73 \pm 0.18$ & $49.24 \pm 0.22$ & $53.92 \pm 0.22$ \\
DPP & $48.96 \pm 0.21$ & \third{$28.90 \pm 0.17$} & \third{$36.44 \pm 0.19$}  & $64.15 \pm 0.18$ & \second{$54.18 \pm 0.20$} & \third{$57.86 \pm 0.19$} \\
Coreset & $47.74 \pm 0.20$ & \second{$29.19 \pm 0.18$} & $33.71 \pm 0.18$ & $61.28 \pm 0.18$ & $30.98 \pm 0.19$ & $45.74 \pm 0.23$ \\
Typiclust & \second{$55.65 \pm 0.18$} & $27.45 \pm 0.17$ & $35.46 \pm 0.18$ & \third{$64.16 \pm 0.18$} & $46.70 \pm 0.21$ & $57.83 \pm 0.21$ \\
ProbCover & \third{$52.07 \pm 0.17$} & $ 23.34 \pm 0.11$ & \second{$37.29 \pm 0.18$} & \second{$64.66 \pm 0.18$} & $40.01 \pm 0.21$ & $ 45.32 \pm 0.22$ \\
GMM (Ours)& \first{$58.82 \pm 0.24$}  & \first{$33.34 \pm 0.24$} & \first{$37.68 \pm 0.19$} & \first{$67.18 \pm 0.18$} & \first{$54.35 \pm 0.20$} & \first{$59.05 \pm 0.20$} \\
\Xhline{2\arrayrulewidth}
\end{tabu}
\vspace{2mm}
\caption{5-Way K-Shot on MiniImageNet with MAML, with $\Pick_\theta^\train$ random.}
\label{tbl:maml_mini}
\end{table*}

\subsubsection{Comparison between active learning methods.}
\cref{fig:al_comp} (left) visualizes context set selection using t-SNE~\cite{tsne08vandermaaten} for one 5-way, 1-shot, unstratified task.
It is vital to select one sample from each class;
only GMM does so here.
\Cref{fig:al_comp} (right) summarizes behavior across many tasks;
while not perfect, GMM does a much better job of selecting distinct classes.

\textbf{Entropy} and \textbf{Margin} are typically far worse than random.
So is \textbf{Coreset}, agreeing with prior observations~\cite{badge2019ash,typiclust2022highlow,probcover2022yehuda}; this may be because of issues with the greedy algorithm and/or sensitivity to outliers.
\textbf{Typiclust} tends to pick points which, while dense according to its ``typicality measure,'' are far from cluster centers; this may be helpful in traditional active learning, but seems to hurt here.
\textbf{DPP} is often better than random, but only barely.

\textbf{ProbCover} manages to cover the feature space well, and is usually second-best. 
However, its ``hard'' radius causes issues; it may be preferable to use a smoother notion, as in GMM. 
The ``purity'' heuristic to choose $\delta$ also does not seem to align well with performance for meta-learning, as shown in \cref{app:sec:prob_cover}.
\Cref{app:sec:further_low_budget} further analyzes the poor performance of other methods.

\textbf{GMM} provides robust performance with few new hyperparameters.\footnote{We did not significantly tune $k$-means or EM parameters from standard defaults.}

``Soft'' $k$-means would be a special case of GMM with a spherical covariance.
For some cases, standard $k$-means performs about the same as GMM, but GMM is occasionally much better: for Baseline++ on CUB, GMM outperforms $k$-means by $3.95$ points for 5-way 1-shot and $11.79$ for 5-shot.
We provide a more thorough comparison to $k$-means in \cref{app:sec:kmeans}.

\begin{table*}[t!]
\centering
\fontsize{8.0}{12.0}\selectfont
\begin{tabu}{c|c|c|c|c}
\hline
\multirow{2}{*}{$\Pick_\theta^\eval$} & \multicolumn{2}{c|}{1-Shot} & \multicolumn{2}{c}{5-Shot}  \\
\cline{2-5}
 & Test strat. & Test unstrat. & Test strat. & Test unstrat. \\
\hline
\hline
Random & $68.44 \pm 0.92$ & $51.03 \pm 0.88$ & $82.66 \pm 0.56$ & \second{$79.57 \pm 0.67$} \\
Entropy & $66.33 \pm 0.91$ & $45.31 \pm 0.89$ & $80.97 \pm 0.60$ & $78.33 \pm 0.72$ \\
Margin& $68.65 \pm 0.90$ & $50.48 \pm 0.94$ & \third{$82.29 \pm 0.64$} & $71.07 \pm 0.83$  \\
DPP & \third{$71.53 \pm 0.89$} & $54.38 \pm 0.92$ & \first{$82.81 \pm 0.55$} & \third{$78.62 \pm 0.76$}\\
Coreset & $69.01 \pm 0.91$ & \second{$56.22 \pm 0.94$} & $82.07 \pm 0.55$ & $76.35 \pm 0.74$ \\
Typiclust & $70.58 \pm 0.81$ & $29.80 \pm 0.32$ & $74.86 \pm 0.81$ & $70.00 \pm 0.92$ \\
ProbCover & \second{$78.11 \pm 0.69$} & \third{$55.09 \pm 0.98$} & $78.59 \pm 0.64$ & $65.71 \pm 0.97$ \\
GMM (Ours)& \first{$79.98 \pm 0.60$} & \first{$59.55 \pm 0.87$} & \second{$82.55 \pm 0.58$} & \first{$82.68 \pm 0.57$} \\
\Xhline{2\arrayrulewidth}
\end{tabu}
\vspace{2mm}
\caption{5-Way K-Shot on CUB with Baseline++, with $\Pick_\theta^\train$ random.}
\vspace{-1mm}
\label{tbl:baseline_cub}
\end{table*}

\begin{figure}[t!]
    \centering
    \includegraphics[width=0.95\textwidth]{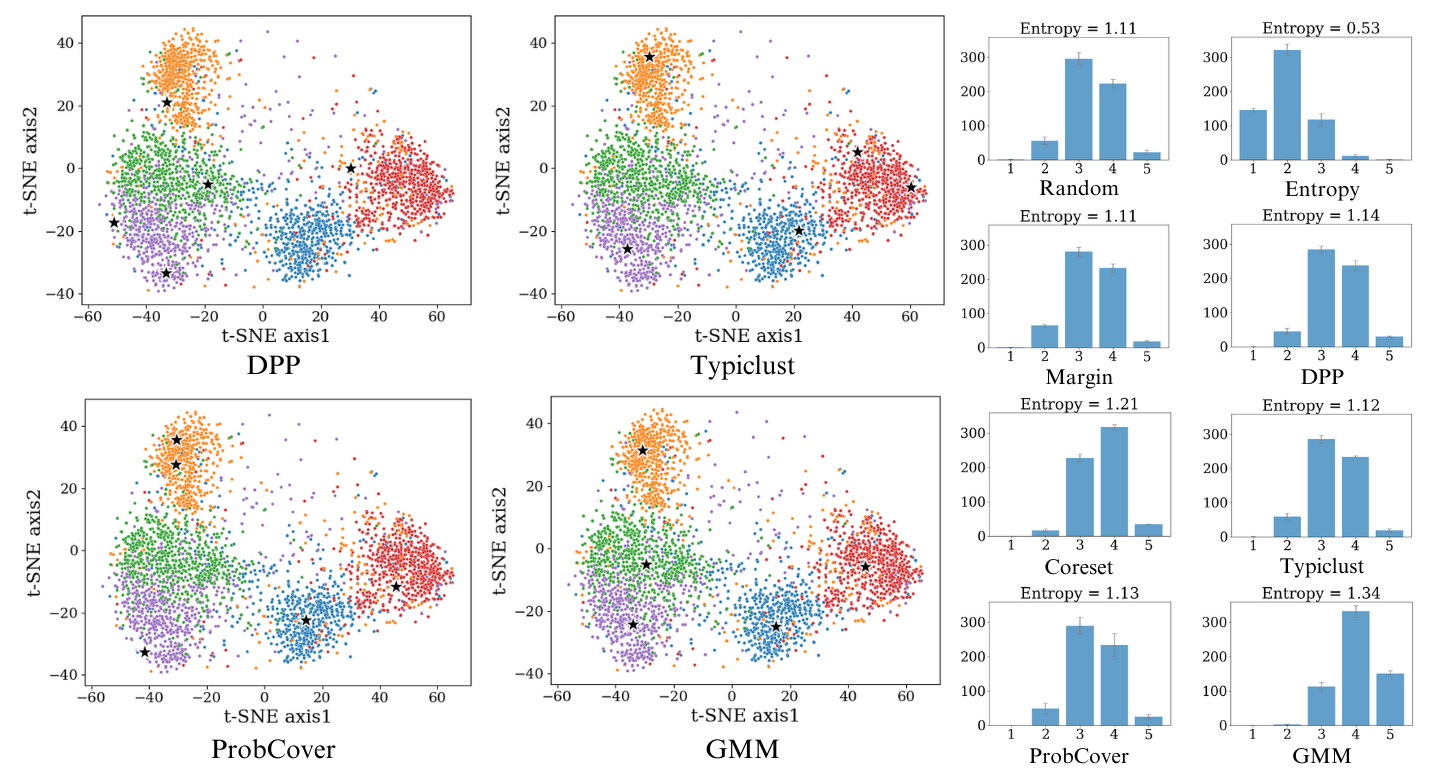}
    \caption{\textbf{Left.} t-SNE of unlabeled points of one 5-way, 1-shot, unstratified MiniImageNet task. Stars denote selected context points using each method. \textbf{Right.} Distributions of the number of classes selected in each $\widetilde\C$ by ProtoNet on MiniImageNet among 600 meta-test cases, along with the mean empirical entropy of $y$ from $\widetilde\C$. 
    The higher the value is, the more diverse classes are selected; $\log 5 \approx 1.6$ would be perfect.
    }
    \label{fig:al_comp}
    \vspace{3mm}
\end{figure}

\begin{table*}[t!]
\begin{minipage}{\textwidth}
\centering
\fontsize{8.5}{13}\selectfont
\begin{tabu}{c|c|c|c|c|c}
\hline
\multirow{2}{*}{Data \& Model} & \multirow{2}{*}{Clustering} & \multicolumn{2}{c|}{1-Shot} & \multicolumn{2}{c}{5-Shot} \\
\cline{3-6}
& &  Train strat. & Unstrat. & Train strat. & Unstrat. \\
\hline
\hline
\multirow{4}{*}{\makecell{MiniImage. \\ MAML}} & 
Weighted Ent. & 22.69 $\pm$ 0.18 & 32.27 $\pm$ 0.32 & 23.75 $\pm$ 0.25 & 46.80 $\pm$ 0.33 \\
& BADGE & 27.71 $\pm$ 0.18 & \second{$34.30 \pm 0.21$} & \second{$41.37 \pm 0.28$} & \second{$58.79 \pm 0.24$} \\ 
& $k$-means Ent. & \second{$30.59 \pm 0.28$} & 33.73 $\pm$ 0.24 & 38.24 $\pm$ 0.29 & 54.87 $\pm$ 0.26 \\
& GMM (Ours) & \first{$33.34 \pm 0.24$} &  \first{$37.68 \pm 0.19$} & \first{$54.35 \pm 0.20$} & \first{$59.05 \pm 0.20$} \\
 \hline
 \multirow{4}{*}{\makecell{FC100 \\ ProtoNet}} &
 Weighted Ent. & \second{$31.80 \pm 0.20$} & 28.94 $\pm$ 0.19 & 40.40 $\pm$ 0.25 & 39.95 $\pm$ 0.25 \\
 & BADGE & 30.91 $\pm$ 0.23 & 29.29 $\pm$ 0.28 & \second{$43.85 \pm 0.22$} & \second{$44.00 \pm 0.29$} \\
 & $k$-means Ent. & 30.93 $\pm$ 0.22 & \second{$30.43 \pm 0.24$} & 41.76 $\pm$ 0.27 & 43.41 $\pm$ 0.29 \\
 & GMM (Ours) & \first{$34.23 \pm 0.23$} & \first{$35.03 \pm 0.23$} & \first{$46.30 \pm 0.21$} & \first{$47.03 \pm 0.20$} \\
\Xhline{2\arrayrulewidth}
\end{tabu}
\end{minipage}%
\vspace{2mm}
\caption{Comparison of GMM with hybrid active learning methods.}
\label{app:tbl:hybrid_vs_gmm}
\end{table*}

\subsection{Comparison with Hybrid Active Learning Methods}
\label{subsec:hybrid}

We compare the proposed GMM method with ``hybrid'' methods that select data points for annotation using both uncertainty and representation measures.\footnote{We separate the comparison with hybrid to highlight it, because hybrid methods are often considered better than solely uncertainty- or representation-based methods.} 
\begin{description}[itemsep=2pt,parsep=2pt,leftmargin=!,labelindent=1em,labelwidth=5em,itemindent=4em,labelsep=1ex,]
    \item[Weighted Entropy] Nguyen~\etal\cite{weigted_logistic2004nguyen} propose weighted expected error for binary classification. For multi-class cases, we derive that it becomes weighted entropy, where weights are likelihood computed using soft $k$-means.
    \item[BADGE] This method \cite{badge2019ash} selects points using $k$-means$^{++}$ with embeddings derived from the gradients of loss w.r.t the weights of the last layer.
    \item[$k$-means Entropy] This approach \cite{hybrid_active_meta2021al} first clusters unlabeled samples using $k$-means$^{++}$, then selects samples per cluster using the classifier's entropy.
\end{description}
\cref{app:tbl:hybrid_vs_gmm} shows that the proposed GMM-based method significantly outperforms all the hybrid methods.
This experiment, along with the poor performance of uncertainty methods such as Entropy, demonstrates that for 
the very-low-budget regime, diversity is significantly more important than reducing uncertainty. 

\subsection{Cross-Domain Active Meta-Learning}
\label{subsec:few_shot_cross_domain}
Cross-domain learning, where $\mathcal P^\train$ is ``fundamentally different'' from $\mathcal P^\eval$, is typically more difficult than ``in-domain'' meta-learning.
We use a ResNet18~\cite{Resnet2015He} pretrained with standard supervised learning on ImageNet,
and meta-test on CUB and \textbf{Places}~\cite{places2017zhou},
which contains images of ``places'' such as restaurants.
As used for cross-domain meta-learning by \cite{cross_domain2022oh}, it contains $16$ classes with an average of $1{,}715$ images each.
As the model is not meta-trained,
train stratification is not relevant;
we show results in \cref{tbl:cross_domain} only for unstratified test sets.
GMM is again the clear overall winner; other methods are often worse than random.

\begin{table*}[t!]
\centering
\fontsize{8.0}{12.0}\selectfont
\begin{tabu}{c|c|c|c|c}
\hline
\multirow{3}{*}{$\Pick_\theta^\eval$} & \multicolumn{2}{c|}{$P^\eval$ on Places} & \multicolumn{2}{c}{$P^\eval$ on CUB} \\
\cline{2-5}
& 1-Shot & 5-Shot  &  1-Shot & 5-Shot  \\
 \cline{2-5}
\hline
\hline
Random & $44.28 \pm 1.93$ & $77.92 \pm 1.70$ & $49.93 \pm 0.92$ & \second{$84.38 \pm 0.72$}  \\
Entropy & $36.12 \pm 1.25$ & $57.79 \pm 2.93$ & $41.85 \pm 0.99$ & $71.15 \pm 0.99$ \\
Margin & $43.31 \pm 1.97$ & $73.65 \pm 1.94$ & $48.04 \pm 0.98$ & $78.84 \pm 0.92$  \\
DPP & $46.76 \pm 2.29$  & \second{$78.36 \pm 1.89$} & \third{$51.41 \pm 0.90$} & \third{$84.19 \pm 0.72$}  \\
Coreset & \second{$50.03 \pm 0.93$} & $65.20 \pm 2.77$ & $50.77 \pm 0.95$ & $81.80 \pm 0.81$ \\
Typiclust & $43.76 \pm 1.98$ & \third{$77.57 \pm 1.84$}  & $43.39 \pm 1.03$ & $50.69 \pm 1.08$ \\
ProbCover & \third{$47.93 \pm 1.08$} & $59.08 \pm 2.50$ & \first{$62.13 \pm 1.08$} & $69.80 \pm 1.16$ \\ %
GMM (Ours)& \first{$60.01 \pm 0.86$} & \first{$86.45 \pm 1.42$} & \second{$59.87 \pm 0.86$} & \first{$85.49 \pm 0.67$} \\
\Xhline{2\arrayrulewidth}
\end{tabu}
\vspace{2mm}
\caption{Cross-domain meta-learning tasks using a ResNet18 pre-trained on ImageNet.}
\label{tbl:cross_domain}
\end{table*}

\subsubsection{Discussion}
Most uncertainty measures tend to be high near decision boundaries. This may be sub-optimal in low-budget settings, as these uncertain points often represent outliers, or are too challenging to generalize.

The primary purpose of context sets in meta-learning is to inform predictions on target samples, necessitating the selection of easily referable points. 
If the selected context samples are too distant from the target samples, making accurate predictions for the target set becomes difficult. 
Diversity measures, particularly GMM, ensure that the context set remains close to the target set even in adverse scenarios, such as when target samples are outliers (see \cref{fig:max_margin}).
Thus, it is preferable to solely consider diversity for active selection of context sets. 

While hybrid methods that incorporate both uncertainty and diversity may be beneficial in mid or high-budget active learning scenarios, they provide limited assistance in extremely low-budget scenarios such as meta-learning.

\subsection{Active Meta-Learning for Regression}
\label{subsec:regression}

Each \textbf{sinusoidal function}~\cite{maml2017finn} has task $y = a \sin(x + p)$, where $a \sim \text{Unif}(0.1, 5)$ is the amplitude, and $p \sim \text{Unif}(0, \pi)$ is the phase of sine functions; we use MAML for this dataset. 
\textbf{Distractor} and \textbf{ShapeNet1D} are vision regression datasets \cite{meta_regression2022gao};
the task is to predict the position of a specific object in an image ignoring a distractor,
or to predict an object's 1D pose (azimuth rotation).
\textbf{IC} uses objects whose classes were observed during meta-training, while \textbf{CC} has novel object classes. %
We use conditional Neural Processes (NP) for Distractor, and attentive NP for ShapeNet1D.
Details are provided in \cref{app:sec:additional_exp_reg}.

\Cref{tbl:regression} compares active strategies on these datasets; GMM again performs generally the best, followed by Coreset and DPP instead of ProbCover.

\begin{table*}[t!]
\begin{minipage}{\textwidth}
\centering
\fontsize{8.0}{11.0}\selectfont
\begin{tabu}{c|c|c|c|c|c}
\hline
\multirow{2}{*}{Active Strategy} & \multirow{2}{*}{Sine (3-Shots)} & \multicolumn{2}{c|}{Distractor (2-Shots)} & \multicolumn{2}{c}{ShapeNet1D (2-Shots)} \\
\cline{3-6}
  & & IC & CC  & IC & CC \\
\hline
\hline
Random & $24.17 \pm 0.43$ & $18.91 \pm 2.13$ & $ 25.79 \pm 2.17 $  & $16.52 \pm 1.08$ & $19.07
\pm 1.30$  \\  
DPP & \third{$23.19 \pm 0.51$} & \second{$18.08 \pm 2.12$} & \first{$19.68 \pm 1.92$} & \third{$11.83 \pm 0.85$} & \third{$13.68 \pm 0.93$}\\
Coreset & $31.36 \pm 0.48$ & \third{$19.58 \pm 1.95$} & \third{$24.08 \pm 2.19$} & \second{$11.39 \pm 0.91$} & \second{$13.05 \pm 1.18$}\\
Typiclust & \second{$21.59 \pm 0.40$} & $20.27 \pm 2.15$ & $ 24.96 \pm 2.68$ & $12.54 \pm 1.08$ & $14.58 \pm 1.24$ \\
ProbCover & $29.36 \pm 0.49$ & $21.96 \pm 2.45$ & $25.25 \pm 2.78$ & $12.31 \pm 0.85$ & $13.95 \pm 1.08$ \\
GMM (Ours)& \first{$18.09 \pm 0.38$} & \first{$17.95 \pm 2.05$} & \second{$22.03 \pm 2.42$} & \first{$10.78 \pm 0.72$} & \first{$12.35 \pm 0.97$} \\
\Xhline{2\arrayrulewidth}
\end{tabu}
\end{minipage}%
\vspace{2mm}
\caption{Meta-learning for regression on a toy dataset and two pose estimation datasets for Intra-Category (IC) and Cross-Category (CC). Sine func.\ and {Distractor} use mean squared error, {ShapeNet1D} uses cosine-sine-distance; lower values are better for each.}
\label{tbl:regression}
\end{table*}

\section{Conclusion}
\label{sec:conclusion}

We clarified the ways in which active learning can be incorporated into meta-learning.
While active context set selection does not seem to work at meta-train time (\cref{app:sec:train_time}), it can be extremely useful at meta-testing/deployment time.

We proposed a surprisingly simple method that substantially outperforms previous proposals.
It is intuitive, very easy to implement,
and bears theoretical guarantees in a particular ``stylized'' but informative situation.

\section*{Acknowledgements}
This work was enabled in part by support provided by the
Natural Sciences and Engineering Research Council of Canada,
the Canada CIFAR AI Chairs program,
Mitacs through the Mitacs Accelerate program,
Calcul Québec,
the BC DRI Group,
and the Digital Research Alliance of Canada.

\bibliographystyle{splncs04}
\bibliography{main}

\clearpage
\appendix

\section{Details for Max-Margin Motivation}
\label{app:sec:max-margin}
The following optimization problem is one form of an $N$-class max-margin problem, i.e.\ a multi-class support vector machine \cite{multiclass-svm}, on a training set $\{ (x_i, y_i) \}_{i=1}^m$:
\begin{equation} \label{eq:multiclass-svm}
    \min_{w_1, \dots, w_N} \sum_{y=1}^N \lVert w_y \rVert^2
    \;\text{ s.t. } \;
    \forall i \in [m], \; 
    \forall y' \ne y_i, \;
    w_{y_i}\tp x_i \ge w_{y'}\tp x_i + 1
.\end{equation}
This is a ``hard'' version of the problem used as a classification head by MetaOptNet \cite{metaopt2019lee},
and can be obtained in their framework by taking the penalty parameter $C \to \infty$.

The decision boundaries obtained by small-step-size gradient descent for linear predictors with cross-entropy loss on separable data converge to those obtained by \eqref{eq:multiclass-svm},
as shown by Theorem 7 in Soudry~\etal\cite{soudry:implicit}, for almost all datasets.
Thus, ANIL \cite{anil2019raghu}, which uses gradient descent for linear predictors with cross-entropy loss on separable data,
will approximately obtain the same solution when using enough steps with appropriately small learning rates.

MetaOptNet uses the homogeneous predictors discussed here.
We can handle non-homogeneous linear predictors ($w\tp x + b$ instead of just $w\tp x$)
with the standard trick of adding a constant $1$ feature to each data point.
This solution actually does not quite maximize the margin on the original problem, since it effectively adds $b^2$ to the objective in \eqref{eq:multiclass-svm},
but ANIL will find exactly this same solution when using gradient descent on a function with a separate intercept.

As visualized in \cref{fig:max_margin} and explained in \cref{subsec:gmm}, if the class-conditional data distributions are isotropic Gaussians with the same covariance matrices, it is more advantageous to label the cluster centers than a random point from each cluster (supported by \cref{prop:max-margin}).
We provide the proof for \cref{prop:max-margin} below.

\combined*
\begin{proof}
    Combine \cref{lemma:orthonormal-maxmarg,lemma:iso-lda} below.
\end{proof}

The orthonormal assumption keeps the proof tractable;
far more analysis would be needed without it.
With high-dimensional meta-learned features that are well-aligned to the learning problem, however, it is reasonable to expect that inner products between different classes will be much smaller than the within-class inner products.

\begin{figure*}[ht]
    \centering
    \begin{subfigure}[b]{0.45\textwidth} %
        \includegraphics[width=\textwidth]{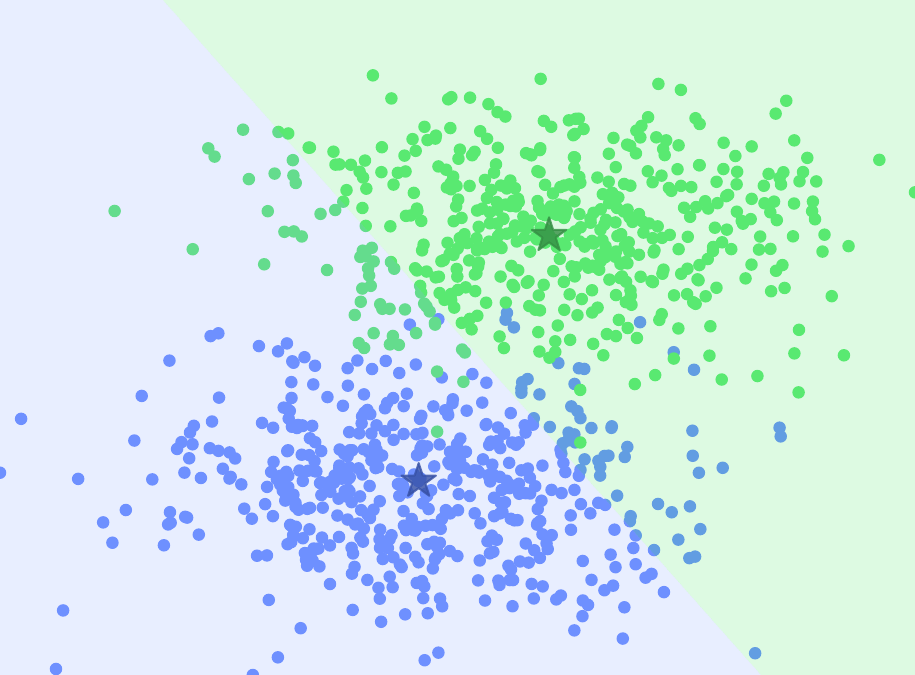}
        \caption{Trained on $N$ cluster centers}
        \label{app:fig:max_margin_aniso:one}
    \end{subfigure}
    \hfill
    \begin{subfigure}[b]{0.45\textwidth} %
        \includegraphics[width=\textwidth]{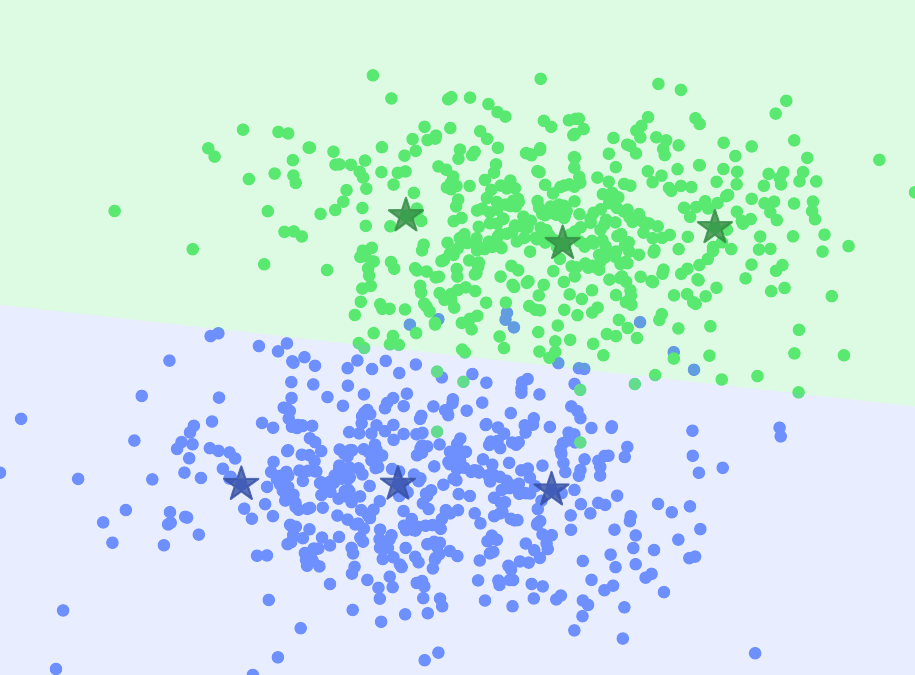}
        \caption{Trained on $3N$ cluster centers}
        \label{app:fig:max_margin_aniso:three}
    \end{subfigure}
    \caption{Decision boundaries using a multiclass SVM \eqref{eq:multiclass-svm} trained on cluster centers (shown by stars), with (a) the one-shot case and (b) the three-shot case.}
    \label{app:fig:max_margin_aniso}
\end{figure*}

This optimality result can break when the clusters do not share a spherical covariance;
consider \cref{app:fig:max_margin_aniso:one},
where the data is still Gaussian
but the shared class-conditional covariance is not spherical.
In the one-shot case, max-margin on the separators
does not choose the optimal separator.
In this case, we could manually select points to choose the correct line.
Doing so, however, is quite risky;
since we do not know the data labels (or that it is actually Gaussian),
we might incorrectly separate the data.
\Cref{app:fig:max_margin_aniso:three} shows
the same problem in a three-shot setting;
here, even though the data is truly generated from a mixture of two Gaussians,
fitting a mixture of six Gaussians gives us an approximate set cover of the data,
and the max-margin separator now works well.

In fact, we can expect that 
(a) as the number of clusters grows,
the cluster centers produce a better and better set cover of the dataset;
(b) the max-margin separator on a set cover will approximate the max-margin separator on the full dataset,
since the support vectors are all nearby.
\danica{would be nice to prove this\dots sigh}

\subsection{Proofs}

\maxmargin*

\begin{proof}
    We will be able to analytically solve the KKT conditions for \eqref{eq:multiclass-svm} in this case.
    Rather than using existing analyses of \eqref{eq:multiclass-svm},
    it will be simpler to directly analyze this particular case.

    Let $\w = \begin{bmatrix} w_1 \\ \vdots \\ w_N \end{bmatrix} \in \R^{N d}$,
    where $d$ is the dimension of the $x_i$ and $w_y$.
    The objective of our optimization problem is then simply $\lVert \w \rVert^2$.

    We will next define a matrix $A$ such that the constraints can be written as $A \w + \one \le \zero$,
    with
    $A \in \R^{N (N-1) \times N d}$ and
    $\le$ interpreted elementwise.
    Each constraint is of the form $- w_i\tp x_i + w_j\tp x_i + 1 \le 0$,
    where $i \ne j$ are class indices in $[N]$.
    We can write the corresponding row of $A$ as
    $(E_j - E_i) x_i$,
    where $E_i \in \R^{N d \times d}$ are given by 
    $E_i = \begin{bmatrix} 0_{(i-1) d \times d} \\ I_d \\ 0_{(N-i-1) d \times d} \end{bmatrix}$;
    these $E_i$ are a block-matrix analogue of standard basis vectors,
    so that $E_i x_i \in \R^{N d}$ has $x_i$ in the $i$th block of $d$ coordinates, and $0$ elsewhere.
    We will order these constraints in $A$ in ``row-major'' order:
    recalling that $i \ne j$,
    this means we have first $i=1 \; j=2$, then $i=1 \; j=3$, up to $i=1 \; j=N$,
    followed by $i=2 \; j=1$, $i=2 \; j=3$, and so on.
    Let $\ell(i, j)$ give the index of the corresponding constraint, so that e.g.\ $\ell(1, 3) = 2$.

    Now, the problem can be written
    \[
        \min_{\w \in \R^{N d}} \frac12 \lVert \w \rVert^2
        \;\text{s.t.}\; A \w + \one \le \zero
    ,\]
    with the $\frac12$ introduced for convenience.
    The KKT conditions for this problem are
    \[
        \w + A\tp \mu = \zero
        \quad
        A \w + \one \le \zero
        \quad
        \mu \ge \zero
        \quad 
        \mu \odot (A \w + \one) = \zero
    ,\]
    where $\odot$ is elementwise multiplication.
    From the first condition, $\w = -A\tp \mu$, where $\mu \in \R^{N (N-1)}$ is any vector satisfying
    \[
        \mu \ge \zero
        \quad
        A A\tp \mu - \one \ge \zero
        \quad
        \mu \odot (A A\tp \mu - \one) = \zero
    .\]
    Since \eqref{eq:multiclass-svm} is a strictly convex minimization problem with affine constraints,
    these conditions are necessary and sufficient for optimality,
    and the solution $\w$ is unique.

    We can reasonably expect, since the $x_i$ are orthonormal, that all constraints should be active,
    meaning that $A A\tp \mu = \one$.
    Indeed, choosing $\mu = (A A\tp)^{-1} \one$ automatically satisfies the second and third conditions;
    it only remains to show that this $\mu \ge \zero$
    in order to show this as an optimal solution to \eqref{eq:multiclass-svm}.

    To do this, we will explicitly characterize $A A\tp$:
    \[
        (A A\tp)_{\ell(i, j), \ell(i', j')}
        = x_i\tp (E_j - E_i)\tp (E_{j'} - E_{i'}) x_{i'}
        = (\delta_{i i'} + \delta_{j j'} - \delta_{i j'} - \delta_{j i'}) \, x_i\tp x_{i'}
    ,\]
    where $\delta_{ij} = \indic(i = j)$ is the Kronecker delta,
    since $E_i\tp E_j = \delta_{ij} I_d$.
    
    Since the $x_i$ are orthonormal,
    $x_i\tp x_{i'} = \delta_{i i'}$.
    As we know $i \ne j$ and $i' \ne j'$,
    this simplifies to
    \[
        (A A\tp)_{\ell(i,j), \ell(i',j')}
        = \delta_{i i'} (1 + \delta_{j j'})
    .\]
    Thus $(A A\tp)$ is a block matrix
    with diagonal blocks of size $(N-1) \times (N-1)$
    with values $I_{N-1} + \one_{N-1} \one_{N-1}\tp$,
    and all off-diagonal blocks zero.
    Taking $\mu = (A A\tp)^{-1} \one_{N (N-1)}$,
    the zero blocks contribute nothing,
    so each block of $N-1$ entries of $\mu$
    is $(I_{N-1} + \one_{N-1} \one_{N-1})^{-1} \one_{N-1}$.
    
    Note that $\one_{N-1} \one_{N-1}\tp$ has one eigenvector $v_1 = \frac{1}{\sqrt{N-1}} \one$ with eigenvalue $\lambda_1 = N - 1$,
    and the remaining eigenvalues are all zero with eigenvectors satisfying $v_i\tp \one = 0$.
    Adding $I$ to this matrix simply increases all eigenvalues by one.
    Thus,
    \begin{align}
      \left( I + \one \one\tp \right)^{-1} \one
    = &\frac{1}{N}
      \left(\frac{1}{\sqrt{N - 1}} \one \right)
      \left(\frac{1}{\sqrt{N - 1}} \one \right)\tp
      \one \\
    &+ \sum_{i = 2}^{N-1} v_i \underbrace{v_i\tp \one}_{0}
    = \frac1N \underbrace{\frac{\one \tp \one}{N - 1}}_1 \one
    = \frac1N \one
    ,
    \end{align}
    and so $\mu = \frac1N \one_{N (N-1)}$,
    which is indeed $\ge \zero$;
    thus this is an optimal solution to the problem.

    We next reconstruct $\w = - A\tp \mu = - \frac1N A\tp \one_{N (N-1)}$.
    Consider the block $w_i$ inside $\w$;
    its value will be the negative mean of the entries of $A$ with an $E_i$ in them.
    The $\ell(i, j)$ rows for $j \ne i$ contribute $N - 1$ entries of the form $-E_i x_i$.
    We also have the $\ell(k, i)$ rows,
    which have one $E_i x_k$ term for each $k \ne i$.
    Thus
    \[
        w_i
        = -\frac1N \left( - (N - 1) x_i + \sum_{k \ne i} x_k \right)
        = - \frac1N \left( - N x_i + \sum_{k=1}^N x_k \right)
        = x_i - \bar{x}
    ,\]
    where $\bar x = \frac1N \sum_{k=1}^N x_k$.
    Thus, for a test point $x$,
    \[
        \argmax_i w_i\tp x
        = \argmax_i x_i\tp x - \bar x\tp x
        = \argmax_i x_i\tp x
    .\]
    Because the $x_i$ are orthonormal,
    this is further equal to
    \[
        \argmin_i \lVert x_i \rVert^2 + \lVert x \rVert^2 - 2 x_i\tp x
        = \argmin_i \lVert x - x_i \rVert
    .\qedhere\]
\end{proof}

\begin{lemma} \label{lemma:iso-lda}
    If $X \mid Y = y \sim \mathcal N(\mu_y, \sigma^2 I)$
    and $Y \sim \operatorname{Uniform}([N])$,
    the Bayes-optimal classifier is given by
    \[
        f^*(x) = \argmin_{y} \lVert x - \mu_y \rVert
    .\]
\end{lemma}
\begin{proof}
    This well-known fact follows by combining
    \[
        p(Y = y \mid X = x)
        = \frac{p(X = x \mid Y = y) p(Y = y)}{p(X = x)}
        \propto p(X = x \mid Y = y)
    \]
    with the definition of the density for $X$,
    \[
    \argmax_y \frac{1}{(2 \pi \sigma^2)^{d/2}} \exp\left( - \frac{1}{2 \sigma^2} \lVert x - \mu_y \rVert^2 \right)
    = \argmin_y \lVert x - \mu_y \rVert
    .\qedhere \]
\end{proof}

\section{Implementation Details for Meta Learning Algorithms}
\label{app:sec:ml}

\subsubsection{Metric-based} We use a meta learning library called learn2learn~\cite{learn2learn2020Arnold} to implement \textbf{ProtoNet}~\cite{prototypical2017snell}. 
Following the original paper, we train a model with $30$-way and $20$-way for 1-Shot and 5-Shot, respectively, for $3{,}000$ iterations.
We use a 4 layer convolutional neural network (Conv4) with $64$ channel size, and the batch size is set to $100$.
For optimization, we employ an Adam optimizer with a learning rate of $0.01$ without having a learning rate schedule.

\subsubsection{Optimization-based}
We use learn2learn library to implement both \textbf{MAML} \cite{maml2017finn} and \textbf{ANIL}~\cite{anil2019raghu}.
We use Conv4 with $32$ channel size for MAML and $64$ channel size for ANIL (larger channel size does not perform better for MAML). 
We train both MAML and ANIL for $60{,}000$ iterations. 
For optimizer, we employ an Adam optimizer for both with learning rates of $0.003$ and $0.001$ (adaptation learning rates of $0.5$ and $0.1$) for MAML and ANIL, respectively.
Batch sizes are set to $32$ for both.

For \textbf{MetaOptNet}~\cite{metaopt2019lee}, we use the publicly available code provided by the authors of the paper (\url{https://github.com/kjunelee/MetaOptNet}). 
We employ the dual formulation of Support Vector Machine (SVM) proposed in MetaOptNet (MetaOptNet-SVM) for experiments with the training shot of $15$, and use the default hyperparameter settings.
For instance, we use a SGD optimizer with initial learning rate of $0.1$ which decays step-wise. 
We train a model for $60$ epochs with a batch size of $8$.

\subsubsection{Model-based}
For both Conditional Neural Process (CNP)~\cite{cnp2018garnelo2018} and Attentive Neural Process (ANP)~\cite{attnnp2019kim}, we use the publicly available code provided by the authors of the paper that addresses regression tasks for computer vision problems~\cite{meta_regression2022gao} (\url{https://github.com/boschresearch/what-matters-for-meta-learning}).

As the authors provide the model checkpoints for CNP on Distractor dataset and ANP on ShapeNet1D, we utilize them to compare active learning methods in meta-test time.
We use 2-Shot for context sets in meta-test time instead of 25-Shot as done in the original work, since 25-Shot is too large to investigate the difference between active learning methods.

\subsubsection{Pre-training-based}

We use the publicly available code provided by the authors of the papers for both \textbf{Baseline++}~\cite{baseline++2019chen} (\url{https://github.com/wyharveychen/CloserLookFewShot}) and \textbf{SimpleShot}~\cite{simpleshot2019wang} (\url{https://github.com/mileyan/simple_shot}).
For both models, we use the features from the pre-trained models on the whole training dataset in inference time.
As reported in the public repository for Baseline++, the performance on CUB for 1-Shot and 5-Shot is lower than the numbers reported in the paper by about $1.1\%$ and $2.5\%$, respectively.
Similarly, the reproduced performance of SimpleShot for 1-Shot and 5-Shot is lower by about $4\sim 5\%$.
Note that the numbers correspond for the case of fully stratified random sampling.

\section{Relationship with Semi-Supervised Few-Shot learning}
\label{app:sec:semi}
In this section, we describe the relationship between active meta-learning and semi-supervised few-shot learning~\cite{embarrassingly2022wei,semi_fewshot2018ren}.

Both semi-supervised few-shot learning and active meta-learning aim to reduce the cost of manual data annotation, but they approach this goal differently. Semi-supervised few-shot learning leverages unlabeled data points without additional annotation, while active meta-learning iteratively adds new labeled data points selected from an unlabeled pool.

Among semi-supervised few-shot learning approaches, pseudo-labeling~\cite{pseudo2021huang} is particularly closely related to active learning. Both pseudo-labeling and active learning utilize unlabeled data, but their methodologies differ. Active learning uses uncertainty or diversity to select data for oracle labeling, targeting points whose labels are unknown. In contrast, pseudo-labeling uses a trained model to predict data labels, which can introduce errors if predictions are incorrect. Thus, pseudo-labeling focuses on data points where the model is already confident—precisely the points active learning would not select.

Combining these contrasting methods could be beneficial and interesting. Pseudo-labeling requires a well-trained classifier, which active learning can support by providing a robust labeled dataset.

\section{Implementation Details for Active Learning Strategies}
\label{app:sec:al}
In this section, we provide detailed description for the implementation of the following active learning methods.
\begin{description}[itemsep=2pt,parsep=2pt,leftmargin=!,labelindent=0.75em,labelwidth=7.5em,itemindent=8.5em,labelsep=0.25em,]
    \item[DPP~\cite{batch_dpp2019biyik}] We use DPPy library~\cite{dppy19GPBV} to implement DPP selection. Gram matrix of the features from the penultimate layer are used as L-ensembles for DPP. We employ $k$-DPP to select $k$ number of context data points.
    \item[Coreset~\cite{coreset2017sener}] We refer to both original code and code provided by the authors of Typiclust and ProbCover. Since we assume that there is no initial labeled data points, we randomly choose the first data point and then apply the greedy algorithm after that.
    \item[Typiclust~\cite{typiclust2022highlow}] We refer to the publicly available code provided by the authors of the paper (\url{https://github.com/avihu111/TypiClust}). As the maximum number of data points to annotate is $25$ ($= 5$-Way $\times$ $5$-Shot), we do not set the maximum number of clusters unlike the original paper. We set the $k$ in $k$-NN to $20$ as with the original work.
    \item[ProbCover~\cite{probcover2022yehuda}] We use the code provided by the original authors of the paper (it is the same as Typiclust). As we state  in \cref{app:sec:prob_cover} and \cref{app:sec:self_sup}, we exploit the features from the meta learners instead of self-supervised features to determine the radius parameters of ProbCover. In particular, the radius for each algorithm and dataset combination is determined as shown in \cref{app:sec:prob_cover}.
    \item[GMM (Ours)] We refer to a publicly available implementation for GMM (\url{https://github.com/ldeecke/gmm-torch}). As previously mentioned, we initialize the cluster centers using $k$-means. Then, we update the cluster means and covariance matrix (shared by all the clusters) using expectation maximization algorithm for up to $100$ iterations. We make the covariance matrix shared between the clusters because we assume that the ``influence" of each annotated data point to other data points is roughly the same regardless of data point although the weight of each dimension may be different (if they are the same, it is equivalent to $k$-means).
\end{description}

\section{Comparison of quality of selected data points}
\label{app:sec:goodness}
In this section, we estimate the quality of selected data points from the low budget active learning methods.
In \cref{app:fig:al_quality}, we compare them in the distance and accuracy as explained in the caption with ANIL~\cite{anil2019raghu} on MiniImageNet.
Whether a task is 1-Shot or 5-Shot, or train-time stratified or unstratified, we can observe that the metrics for GMM are consistently the best.

\begin{figure*}[t!]
    \centering
    \captionsetup{font=small}
    \begin{subfigure}[b]{0.48\textwidth} %
        \captionsetup{font=small}
        \includegraphics[width=\textwidth]{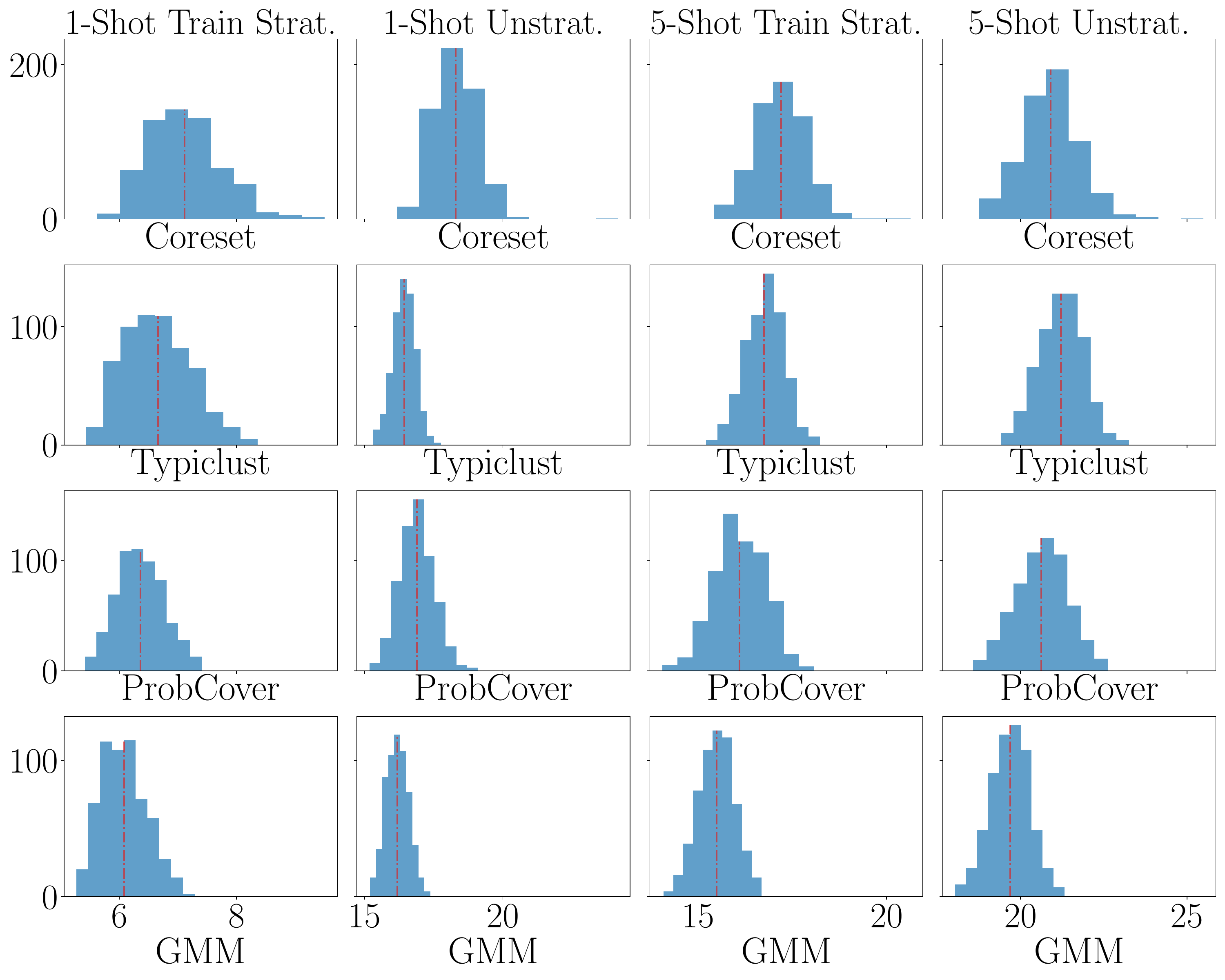}
        \caption{Distributions of distances}
        \label{app:fig:al_quality_distance}
    \end{subfigure}
    \hfill
    \begin{subfigure}[b]{0.48\textwidth} %
        \captionsetup{font=small}
        \includegraphics[width=\textwidth]{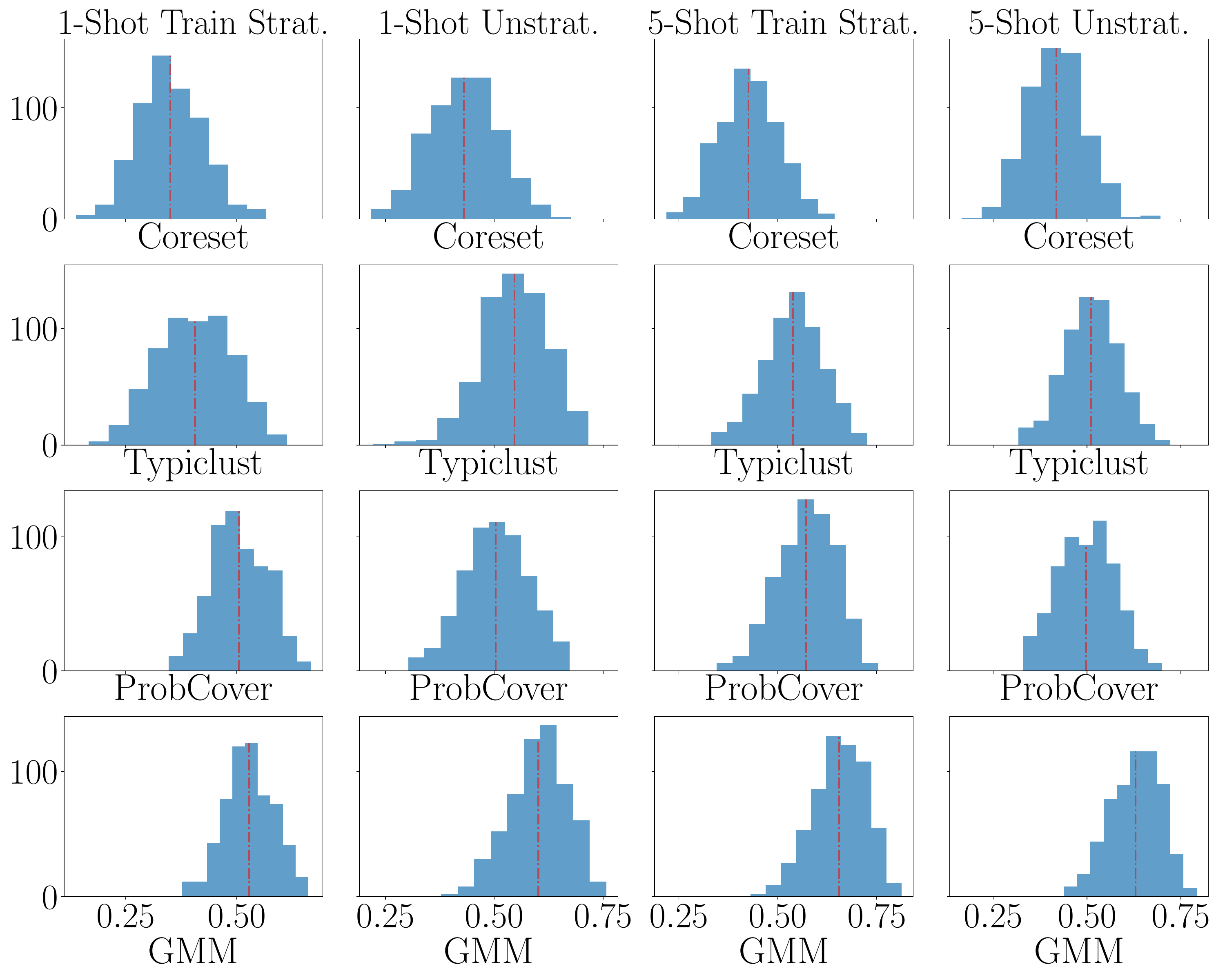}
        \caption{Distributions of accuracies}
        \label{app:fig:al_quality_acc}
    \end{subfigure}
    \caption{Estimation of goodness of selected data points on MiniImageNet with ANIL using the distribution of (a) the distance between the unlabeled points and closest selected points, and (b) the equality between the true labels of unlabeled points and labels of the closest select points. Red dotted lines show mean values.}
    \label{app:fig:al_quality}
\end{figure*}

\newcolumntype{E}{>{\centering\arraybackslash}p{6.3em}}
\begin{table*}[t!]
\begin{minipage}{\textwidth}
\centering
\fontsize{4.8}{6.3}\selectfont
\begin{tabu}{c|c|E|E|E|E|E|E}
\hline
\multirow{2}{*}{Data \& Model} & \multirow{2}{*}{Clustering} & \multicolumn{3}{c|}{1-Shot} & \multicolumn{3}{c}{5-Shot} \\
\cline{3-8}
& & Fully strat. & Train strat. & Unstrat. & Fully strat. & Train strat. & Unstrat. \\
\hline
\hline
\multirow{3}{*}{\makecell{MiniImage. \\ MAML}} & $k$-means  & 56.75 $\pm$ 0.20 & \bftab{33.29 $\pm$ 0.26} & 37.26 $\pm$ 0.18 & 65.76 $\pm$ 0.18 & 41.61 $\pm$ 0.24 & \bftab{59.17 $\pm$ 0.20}   \\
& $k$-means$^{++}$ & 56.12 $\pm$ 0.26 & 32.87 $\pm$ 0.32 & \bftab{38.53 $\pm$ 0.21} & 65.49 $\pm$ 0.21 & 43.61 $\pm$ 0.32 & 58.63 $\pm$ 0.26 \\
 & GMM & \bftab{58.82 $\pm$ 0.24} & \bftab{33.34 $\pm$ 0.24} & 37.68 $\pm$ 0.19 & \bftab{67.18 $\pm$ 0.18} & \bftab{54.35 $\pm$ 0.20} & \bftab{59.05 $\pm$ 0.20} \\
 \hline
 \multirow{3}{*}{\makecell{FC100 \\ ProtoNet}} &
 $k$-means 
 & \bftab{50.20} $\pm$ 0.17 & 29.69 $\pm$ 0.20 & \bftab{35.03 $\pm$ 0.23} & 54.07 $\pm$ 0.17 & 41.42 $\pm$ 0.23 & 41.34 $\pm$ 0.23 \\
 & $k$-means$^{++}$ & 49.91 $\pm$ 0.17 & 27.27 $\pm$ 0.22 & \bftab{34.93 $\pm$ 0.27} & \bftab{54.72 $\pm$ 0.30} & 41.61 $\pm$ 0.39 & 42.64 $\pm$ 0.39 \\
 & GMM & \bftab{50.22 $\pm$ 0.18} & \bftab{34.23 $\pm$ 0.23} & \bftab{35.03 $\pm$ 0.23} &  \bftab{54.76 $\pm$ 0.17} & \bftab{46.30 $\pm$ 0.21} & \bftab{47.03 $\pm$ 0.20} \\
\Xhline{2\arrayrulewidth}
\end{tabu}
\end{minipage}%
\vspace{3mm}
\caption{Comparison of GMM and $k$-Means selections on MiniImageNet and FC100 using MAML and ProtoNet.}
\label{app:tbl:kmeans_vs_gmm}
\end{table*}

\section{Comparison to $k$-Means based methods }
\label{app:sec:kmeans}
\cref{app:tbl:kmeans_vs_gmm} compares the proposed GMM method to $k$-means and $k$-means$^{++}$ since they are closely related.

The performance of GMM, $k$-means and $k$-means$^{++}$ are similar in general but for some cases, GMM is significantly better than the others.
We conjecture it is because some features are more important than the others, and since GMM takes it into account using Mahalanobis distance (instead of Euclidean distance used in $k$-means), it selects data points that represents nearby data points better.

\section{Difficulty of Tuning $\delta$ Parameter for ProbCover}
\label{app:sec:prob_cover}

In Section 3.2 of Yehuda~\etal\cite{probcover2022yehuda}, the authors proposed to tune the radius $\delta$ based on the purity defined as,
\begin{align}
    \pi(\delta) = P(\{x: B_\delta(x) \text{ is pure} \}) \,\,\, \text{ where } \,\,\, B_\delta(x) = \{x': \lVert x' - x \rVert_2 \leq \delta  \}
\end{align}
Here, a ball $B_\delta(x)$ is ``pure" if $f(x') = y, \, \forall x' \in B_\delta(x)$ where $y$ is the label of $x$.
As the radius $\delta$ increases, the purity decreases monotonically.
They choose the optimal radius $\delta^*$ as $\delta^* = \max\{ \delta : \pi(\delta) \geq 0.95 \}$.
More specifically, they first run k-means with k being the number of classes. 
Then, the purity is measured using the k-means assignment as pseudo-labels.

In their setting (pool-based active learning for image classification), since it is hard to obtain meaningful features from a model trained only a few examples, they use the features from self-supervised learning methods such as SimCLR~\cite{simclr2020chen}
It is, however, not the case for meta-learning.
In meta-test time, the features from the meta learner are usually more meaningful than self-supervised learning features.
Hence, we use the mete learner's features to estimate the optimal radius for ProbCover.
Following the original paper, we first run k-means and compute the purity in the same way.
Since the features can differ by meta learning algorithms and the number of shots, we provide the plots for different algorithms as well as 1 and 5-Shots as shown in \cref{app:fig:probcover_radius} (we select the optimal radius $\delta$ based on these plots throughout the experiments).
For \cref{app:fig:probcover_radius}(a)-(f), we also provide the meta-test performance of stratified and unstratified versions of Random selection to demonstrate that the estimated optimal radius and best radius for meta-test accuracy do not align.

Another difficulty of estimating the optimal radius is that it is hard to set a search space for the radius.
As shown in the x-axis of \cref{app:fig:probcover_radius}, the reasonable search space varies significantly depending on the meta-learning algorithms and datasets we use.
In Yehuda~\etal, this was less of a problem since they use SimCLR features, which are normalized: the range of the radius is in $[0, 1]$.
However, as shown in \cref{app:sec:self_sup}, if we use SimCLR features in meta-test time to actively select context sets, the performance generally drops.

\begin{figure*}[t!]
    \centering
    \begin{subfigure}[b]{0.24\textwidth} %
        \includegraphics[width=\textwidth]{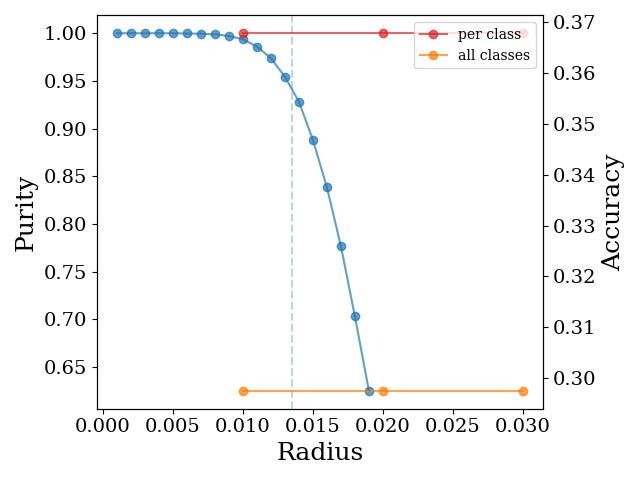}
        \caption{Proto FC100 1-shot}
        \label{fig:proto_fc100_1shot}
    \end{subfigure}
    \hfill
    \begin{subfigure}[b]{0.24\textwidth} %
        \includegraphics[width=\textwidth]{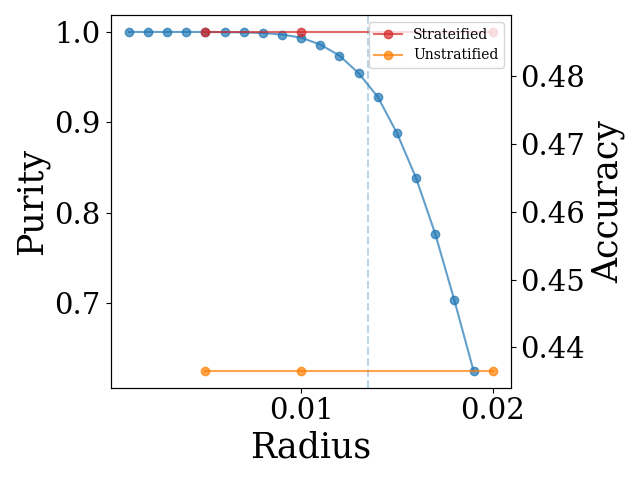}
        \caption{Proto FC100 1-shot}
        \label{fig:proto_fc100_5shot}
    \end{subfigure}
    \hfill
    \begin{subfigure}[b]{0.24\textwidth} %
        \includegraphics[width=\textwidth]{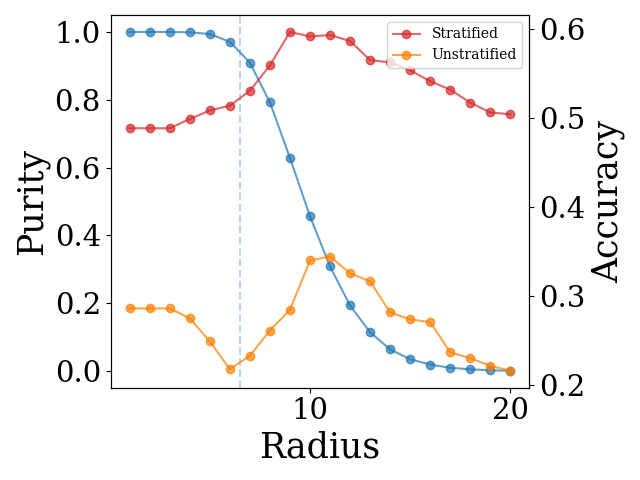}
        \caption{MAML Mini. 1-shot}
        \label{fig:maml_mini_1shot}
    \end{subfigure}
    \hfill
    \begin{subfigure}[b]{0.24\textwidth} %
        \includegraphics[width=\textwidth]{figures/purity/maml_mini_1shot.png}
        \caption{MAML Mini. 5-shot}
        \label{fig:maml_mini_5shot}
    \end{subfigure}
    \hfill
    \begin{subfigure}[b]{0.24\textwidth} %
        \includegraphics[width=\textwidth]{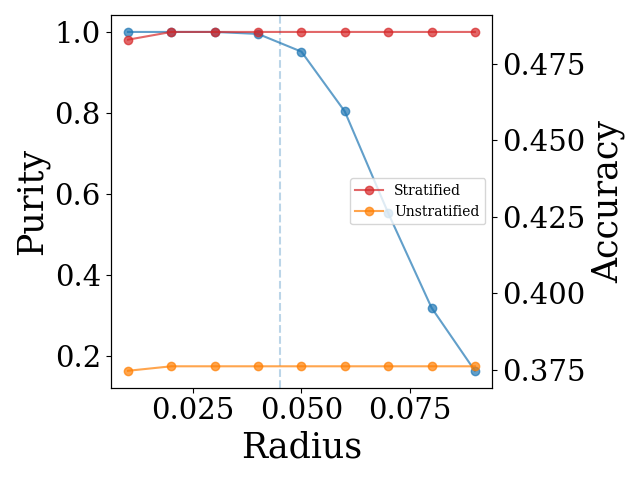}
        \caption{ANIL Tiered. 1-shot}
        \label{fig:anil_tiered_1shot}
    \end{subfigure}
    \hfill
    \begin{subfigure}[b]{0.24\textwidth} %
        \includegraphics[width=\textwidth]{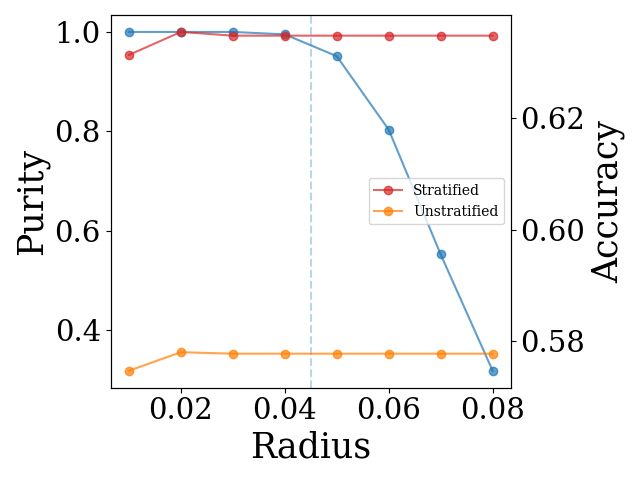}
        \caption{ANIL Tiered. 5-shot}
        \label{fig:anil_tiered_5shot}
    \end{subfigure}
    \hfill
    \begin{subfigure}[b]{0.24\textwidth} %
        \includegraphics[width=\textwidth]{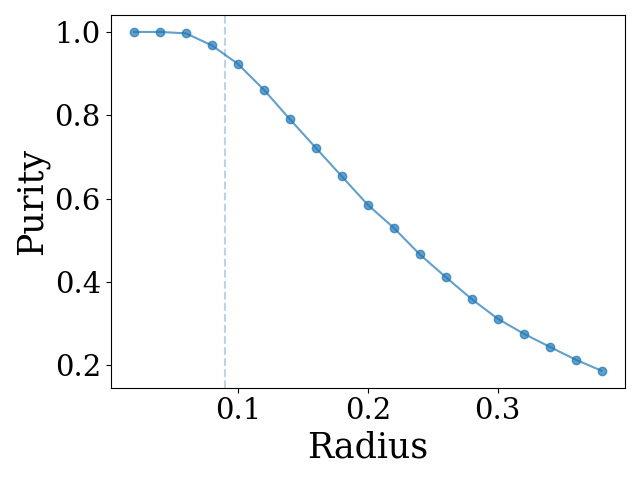}
        \caption{Baseline++ CUB}
        \label{fig:baseline_cub}
    \end{subfigure}
    \hfill
    \begin{subfigure}[b]{0.24\textwidth} %
        \includegraphics[width=\textwidth]{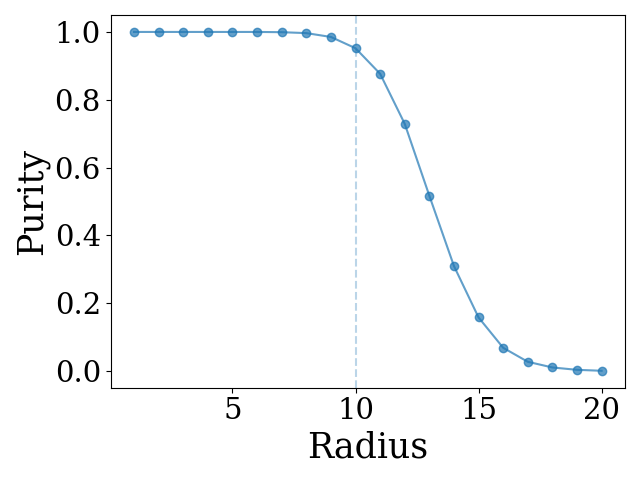}
        \caption{SimpleShot Mini.}
        \label{fig:simpleshot_mini}
    \end{subfigure}
    \caption{Estimation of the optimal radius for ProbCover in meta-learning}
    \label{app:fig:probcover_radius}
\end{figure*}

\section{Additional Experimental Results for Classification}
In this section, we provide addtional experimental results for few-shot image classification.
In \cref{app:tbl:anil_tieredimagenet}, we compare the active learning strategies for ANIL~\cite{anil2019raghu} on the TieredImageNet dataset.
Similarly, \cref{app:tbl:metaopt_fc100} provides the results with MetaOptNet~\cite{metaopt2019lee} on FC100 dataset.
\cref{app:tbl:simpleshot_mini}, \cref{app:tbl:proto_mini}, and \cref{app:tbl:anil_mini} are for SimpleShot~\cite{simpleshot2019wang}, ProtoNet~\cite{prototypical2017snell}, and ANIL~\cite{anil2019raghu} on MiniImageNet, respectively.
Note that Entropy and Margin selections are not applicable for MetaOptNet-SVM.
Regardless of meta-learning algorithm and dataset, GMM significantly outperforms the other active learning methods, and some of them are worse than the Random selection.

\begin{table*}[t!]
\begin{minipage}{\textwidth}
\centering
\fontsize{6.75}{9.75}\selectfont
\begin{tabu}{c|D|D|D|D|D|D}
\hline
\multirow{2}{*}{$\Pick_\theta^\eval$} & \multicolumn{3}{c|}{1-Shot} & \multicolumn{3}{c}{5-Shot}  \\
\cline{2-7}
 & Fully strat. & Train strat. & Unstrat. & Fully strat. & Train strat. & Unstrat. \\
\hline
\hline
Random & $47.55 \pm 0.18$ & \second{$38.19 \pm 0.16$} & $34.79 \pm 0.15$ & \second{$63.84 \pm 0.17$} & \second{$57.92 \pm 0.23$} & \second{$57.56 \pm 0.18$} \\
Entropy & $43.89 \pm 0.16$ & $32.73 \pm 0.16$ & $26.33 \pm 0.14$ & $57.56 \pm 0.18$ & $40.23 \pm 0.18$ & $34.16 \pm 0.17$ \\
Margin & $47.35 \pm 0.17$ & $36.01 \pm 0.14$ & $30.79 \pm 0.14$ & $62.87 \pm 0.17$ & $54.89 \pm 0.24$ & $56.76 \pm 0.17$ \\
DPP & \third{$49.28 \pm 0.17$} & \third{$38.17 \pm 0.15$} & \third{$36.52 \pm 0.15$} & $63.24 \pm 0.19$ & $57.28 \pm 0.21$ & \third{$57.23 \pm 0.18$} \\
Coreset& $47.32 \pm 0.18$ & $36.97 \pm 0.20$ & \second{$40.72 \pm 0.14$} & $56.93 \pm 0.18$ & $47.68 \pm 0.22$ & $52.89 \pm 0.17$\\
Typiclust & \second{$52.95 \pm 0.18$} & $37.21 \pm 0.17$ & $34.05 \pm 0.14$ & $63.13 \pm 0.19$ & $55.84 \pm 0.22$ & $56.76 \pm 0.17$ \\
ProbCover & $48.53 \pm 0.53$ & $37.61 \pm 0.49$ & $34.53 \pm 0.43$ & \third{$63.48 \pm 0.51$} & \third{$57.77 \pm 0.56$} & $57.12 \pm 0.58$ \\
GMM (Ours)& \first{$60.29 \pm 0.19$} & \first{$50.92 \pm 0.22$} & \first{$42.17 \pm 0.17$} & \first{$66.48 \pm 0.18$} & \first{$60.12 \pm 0.24$} & \first{$60.28 \pm 0.17$} \\
\Xhline{2\arrayrulewidth}
\end{tabu}
\end{minipage}%
\vspace{3mm}
\caption{5-Way K-Shot on TieredImageNet with ANIL, with $\Pick_\theta^\train$ random.}
\label{app:tbl:anil_tieredimagenet}
\vspace{-3mm}
\end{table*}

\begin{table*}[t!]
\begin{minipage}{\textwidth}
\centering
\fontsize{6.75}{9.75}\selectfont
\begin{tabu}{c|D|D|D|D|D|D}
\hline
\multirow{2}{*}{$\Pick_\theta^\eval$} & \multicolumn{3}{c|}{1-Shot} & \multicolumn{3}{c}{5-Shot}  \\
\cline{2-7}
 & Fully strat. & Train strat. & Unstrat. & Fully strat. & Train strat. & Unstrat. \\
\hline
\hline
Random & 40.41 $\pm$ 0.74 & \second{$31.96 \pm 0.56$} & \third{$32.76 \pm 0.63$} & \third{$53.11 \pm 0.66$} & \third{$47.73 \pm 0.70$} & \second{$47.48 \pm 0.76$} \\
DPP & 40.47 $\pm$ 0.80 & \third{$30.33 \pm 0.67$} & \second{$33.41 \pm 0.66$} & 51.44 $\pm$ 0.68 & \second{$48.21 \pm 0.67$} & \third{$47.45 \pm 0.68$}  \\
Coreset & 39.20 $\pm$ 0.71 & 27.55 $\pm$ 0.66 & 30.16 $\pm$ 0.69 & 46.80 $\pm$ 0.67  & 24.08 $\pm$ 0.65  & 25.75 $\pm$ 0.72 \\
Typiclust & \second{$45.20 \pm 0.78$} & 26.35 $\pm$ 0.47 & 27.00 $\pm$ 0.43 & 52.39 $\pm$ 0.66 & 23.97 $\pm$ 0.42 & 24.12 $\pm$ 0.39 \\
ProbCover & \third{$41.93 \pm 0.67$} & 26.87 $\pm$ 0.62 & 27.43 $\pm$ 0.48 & \second{$54.36 \pm 0.76$} & 37.00 $\pm$ 0.69 & 38.33 $\pm$ 0.76 \\ %
GMM (Ours)& \first{$51.16 \pm 0.67$} & \first{$40.89 \pm 0.74$} & \first{$41.61 \pm 0.87$} & \first{$60.48 \pm 0.86$} & \first{$52.68 \pm 0.70$} & \first{$51.79 \pm 0.70$}\\
\Xhline{2\arrayrulewidth}
\end{tabu}
\end{minipage}%
\vspace{3mm}
\caption{5-Way K-Shot on FC100 with MetaOptNet, with $\Pick_\theta^\train$ random.}
\label{app:tbl:metaopt_fc100}
\vspace{-3mm}
\end{table*}

\newcolumntype{F}{>{\centering\arraybackslash}p{5.8em}}
\begin{table*}[t!]
\begin{minipage}{\textwidth}
\centering
\fontsize{8.5}{11.5}\selectfont
\begin{tabu}{c|F|F|F|F}
\hline
\multirow{2}{*}{$\Pick_\theta^\eval$} & \multicolumn{2}{c|}{1-Shot} & \multicolumn{2}{c}{5-Shot}  \\
\cline{2-5}
 & Fully strat. & Train strat. & Fully strat. & Train strat. \\
\hline
\hline
Random & $45.15 \pm 0.73$ & $26.28 \pm 0.61$ & \third{$61.22 \pm 0.72$} & \first{$51.89 \pm 0.73$} \\
Entropy & $37.08 \pm 0.75$ & $21.62 \pm 0.37$ & $47.93 \pm 0.74$ & $32.74 \pm 0.60$ \\
Margin & $41.53 \pm 0.73$ & $24.28 \pm 0.51$ & \second{$62.15 \pm 0.70$} &  \third{$50.90 \pm 0.75$}\\
DPP & $44.52 \pm 0.75$ & \third{$26.32 \pm 0.58$} & $60.93 \pm 0.72$ & \second{$51.79 \pm 0.75$} \\
Coreset & \third{$45.85 \pm 0.73$} & \second{$27.04 \pm 0.54$} & $56.48 \pm 0.72$ & $40.39 \pm 0.68$ \\
Typiclust & $44.53 \pm 0.71$ & $22.97 \pm 0.42$ & $34.21 \pm 0.77$ & $20.04 \pm 0.06$ \\
ProbCover & \second{$49.32 \pm 0.71$} & $24.61 \pm 0.52$ & $55.60 \pm 0.66$ & $32.24 \pm 0.67$ \\
GMM (Ours)& \first{$52.77 \pm 0.72$} & \first{$28.17 \pm 0.64$} & \first{$62.64 \pm 0.71$} & $50.40 \pm 0.75$ \\
\Xhline{2\arrayrulewidth}
\end{tabu}
\end{minipage}%
\vspace{3mm}
\caption{5-Way K-Shot on MiniImageNet with SimpleShot, with $\Pick_\theta^\train$ random.}
\label{app:tbl:simpleshot_mini}
\vspace{-1mm}
\end{table*}

\label{app:sec:additional_exp}

\section{Additional Experimental Details for Regression}
\label{app:sec:additional_exp_reg}

Gao~\etal propose the Distractor and ShapeNet1D datasets to compare meta learning algorithms for vision regression tasks.
They evaluate meta learners for intra-category (IC) and cross-category (CC) inputs where CC corresponds to the cross-domain in few-shot image classification.

\textbf{Distractor} consists of $10$ object classes for a training set and $2$ novel classes for CC evaluation.
Each class contains $1,000$ randomly sampled objects from ShapeNetCoreV2~\cite{chang2015shapenet}.
$20\%$ of training set is reserved for IC evaluation. 
In this dataset, each image consists of two objects: the object of interest and a distractor object, which are positioned randomly. The goal is to recognize and locate the object of interest within the image in the presence of a distractor.

\textbf{ShapeNet1D}~\cite{meta_regression2022gao} consists of $27$ object classes for a training set and $3$ object classes for CC evaluation.
Each object class contains $50$ images, and $10$ images are used for IC evaluation.
ShapeNet1D aims to predict the 1D pose, i.e., rotation angle, around the azimuth axis of an object.

To analyze these vision regression tasks, we compare various active learning strategies in the 2-shot setting.
We use CNP for Distractor, NP for ShapeNet1D.
More details about the models can be found in \cref{app:sec:ml}.

\begin{table*}[t!]
\begin{minipage}{\textwidth}
\centering
\fontsize{6.75}{9.75}\selectfont
\begin{tabu}{c|D|D|D|D|D|D}
\hline
\multirow{2}{*}{$\Pick_\theta^\eval$} & \multicolumn{3}{c|}{1-Shot} & \multicolumn{3}{c}{5-Shot}  \\
\cline{2-7}
 & Fully strat. & Train strat. & Unstrat. & Fully strat. & Train strat. & Unstrat. \\
\hline
\hline
Random & 47.70 $\pm$ 0.20 & \second{$39.65 \pm 0.28$} & \second{$38.72 \pm 0.27$} & \second{$64.66 \pm 0.18$} & \third{$57.36 \pm 0.27$} & \first{$57.42 \pm 0.25$} \\
Entropy & 44.33 $\pm$ 0.20 & 36.35 $\pm$ 0.28 & 34.87 $\pm$ 0.27 & 61.23 $\pm$ 0.19 & 49.83 $\pm$ 0.31 & 48.46 $\pm$ 0.32 \\
Margin & 47.07 $\pm$ 0.20 & 37.69 $\pm$ 0.27 & 37.84 $\pm$ 0.28 & 63.79 $\pm$ 0.18 & 55.25 $\pm$ 0.29 & 56.15 $\pm$ 0.27  \\
DPP & 47.90 $\pm$ 0.20 & 39.17 $\pm$ 0.28 & \third{$37.89 \pm 0.26$} & \third{$64.36 \pm 0.19$} & \second{$57.48 \pm 0.26$} & \second{$57.37 \pm 0.25$}  \\
Coreset & 47.86 $\pm$ 0.20 & \third{$39.51 \pm 0.26$} & 37.79 $\pm$ 0.26 & 55.09 $\pm$ 0.20 & 50.14 $\pm$ 0.29 & 50.27 $\pm$ 0.28  \\
Typiclust & \second{$59.51 \pm 0.17$} & 38.47 $\pm$ 0.27 & 37.57 $\pm$ 0.27 & 61.02 $\pm$ 0.19 & 51.82 $\pm$ 0.31 & 52.02 $\pm$ 0.30 \\
ProbCover & \third{$48.51 \pm 0.20$} & 35.25 $\pm$ 0.26 & 34.50 $\pm$ 0.25 & 43.61 $\pm$ 0.19 & 38.63 $\pm$ 0.21 & 38.24 $\pm$ 0.20  \\
GMM (Ours) & \first{$64.50 \pm 0.16$} & \first{$47.88 \pm 0.32$} & \first{$44.71 \pm 0.29$} & \first{$67.03 \pm 0.19$} & \first{$57.55 \pm 0.29$} & \third{$56.44 \pm 0.30$} \\
\Xhline{2\arrayrulewidth}
\end{tabu}
\end{minipage}%
\captionsetup{font=small}
\vspace{3mm}
\caption{5-Way K-Shot on MiniImageNet with ProtoNet, with $\Pick_\theta^\train$ random.}
\label{app:tbl:proto_mini}
\end{table*}

\begin{table*}[t!]
\begin{minipage}{\textwidth}
\centering
\fontsize{6.75}{9.75}\selectfont
\begin{tabu}{c|D|D|D|D|D|D}
\hline
\multirow{2}{*}{$\Pick_\theta^\eval$} & \multicolumn{3}{c|}{1-Shot} & \multicolumn{3}{c}{5-Shot}  \\
\cline{2-7}
 & Fully strat. & Train strat. & Unstrat. & Fully strat. & Train strat. & Unstrat. \\
\hline
\hline
Random & 46.59 $\pm$ 0.19 & 36.70 $\pm$ 0.19 & 34.79 $\pm$ 0.18 & \third{$61.35 \pm 0.19$} & \third{$55.24 \pm 0.20$} & \third{$56.65 \pm 0.19$} \\
Entropy & 44.63 $\pm$ 0.20 & 35.51 $\pm$ 0.18 & 27.35 $\pm$ 0.14 & 55.09 $\pm$ 0.19 & 39.71 $\pm$ 0.20 & 37.45 $\pm$ 0.19 \\
Margin& 46.58 $\pm$ 0.19 & 36.60 $\pm$ 0.19 & 32.46 $\pm$ 0.18 & 55.62 $\pm$ 0.19 & 40.40 $\pm$ 0.20 & 37.67 $\pm$ 0.19 \\
DPP & 47.33 $\pm$ 0.19 & \third{$37.45 \pm 0.17$} & 37.76 $\pm$ 0.18 & 61.08 $\pm$ 0.19 & \second{$56.18 \pm 0.18$} & \second{$57.08 \pm 0.18$}  \\
Coreset & 46.40 $\pm$ 0.21 & \second{$38.37 \pm 0.17$} & \second{$41.34 \pm 0.17$} & 53.74 $\pm$ 0.20 & 47.81 $\pm$ 0.20 & 51.62 $\pm$ 0.19  \\
Typiclust & \second{$54.44 \pm 0.18$} & 36.78 $\pm$ 0.17 & 34.52 $\pm$ 0.19 & 60.87 $\pm$ 0.18 & 52.56 $\pm$ 0.20 & 55.11 $\pm$ 0.19  \\
ProbCover & \third{$51.56 \pm 0.18$} & 27.49 $\pm$ 0.15 & \first{$41.46 \pm 0.17$}  & \second{$61.68 \pm 0.18$} & 53.80 $\pm$ 0.20 & 42.70 $\pm$ 0.22 \\
GMM (Ours) & \first{$58.50 \pm 0.18$} & \first{$48.13 \pm 0.20$} & \third{$40.26 \pm 0.18$} & \first{$65.14 \pm 0.17$} & \first{$59.01 \pm 0.20$} & \first{$61.48 \pm 0.19$}  \\
\Xhline{2\arrayrulewidth}
\end{tabu}
\end{minipage}%
\captionsetup{font=small}
\vspace{3mm}
\caption{5-Way K-Shot on MiniImageNet with ANIL, with $\Pick_\theta^\train$ random.}
\label{app:tbl:anil_mini}
\end{table*}

\section{Comparison to Self-Supervised Features}
\label{app:sec:self_sup}

ProbCover and Typiclust use self-supervised features to actively select new data points to annotate, since there are not enough labeled data to train a classifier to output meaningful features.
Instead, they utilize the features from SimCLR~\cite{simclr2020chen}.
To validate if it is better to use the features from a meta learner than SimCLR in meta-learning, we compare SimCLR features to the features from either MAML or ProtoNet for Typiclust and ProbCover as shown in \cref{app:tbl:self_sup_typiclust} and \cref{app:tbl:self_sup_probcover}.
Here, we use MiniImageNet and FC100 datasets for MAML and ProtoNet, respecitvely as with \cref{tbl:maml_mini} and \cref{tbl:proto_fc100}.
For both Typiclust and ProbCover, although there are a couple of cases where SimCLR features are better, it is significantly worse than MAML and ProtoNet features in general.
It intuitively makes sense because 1) meta learners are trained with large enough data points and 2) it is likely that the information in self-supervised features do not align with that in meta learners.

\newcolumntype{G}{>{\centering\arraybackslash}p{6.0em}}
\begin{table*}[t!]
\begin{minipage}{\textwidth}
\centering
\fontsize{6.0}{9.0}\selectfont
\begin{tabu}{c|c|G|G|G|G|G|G}
\hline
\multirow{2}{*}{Dataset} & \multirow{2}{*}{Features} & \multicolumn{3}{c|}{1-Shot} & \multicolumn{3}{c}{5-Shot}  \\
\cline{3-8}
& & Fully strat. & Train strat. & Unstrat. & Fully strat. & Train strat. & Unstrat. \\
\hline
\hline
\multirow{2}{*}{Mini.} & MAML & \bftab{55.65} $\pm$ \bftab{0.18} & 27.45 $\pm$ 0.17 & \bftab{35.46} $\pm$ \bftab{0.18} & 64.16 $\pm$ 0.18 & \bftab{46.70} $\pm$ \bftab{0.21} & \bftab{57.83} $\pm$ \bftab{0.21} \\
& SimCLR & 44.84 $\pm$ 0.44 & \bftab{27.59} $\pm$ \bftab{0.35} & 34.80 $\pm$ 0.47 & \bftab{65.95} $\pm$ \bftab{0.43} & 36.03 $\pm$ 0.48 & 57.77 $\pm$ 0.47 \\
\hline 
\multirow{2}{*}{FC100} & ProtoNet & \bftab{46.01} $\pm$ \bftab{0.16} & \bftab{30.96} $\pm$ \bftab{0.19} & \bftab{30.61} $\pm$ \bftab{0.21} & 47.54 $\pm$ 0.17 & \bftab{43.61} $\pm$ \bftab{0.18} & \bftab{44.03} $\pm$ \bftab{0.21} \\
& SimCLR & 36.07 $\pm$ 0.44 & 29.60 $\pm$ 0.46 & 30.13 $\pm$ 0.45 & \bftab{48.59} $\pm$ \bftab{0.49} & 43.29 $\pm$ 0.49 & 43.89 $\pm$ 0.59 \\
\Xhline{2\arrayrulewidth}
\end{tabu}
\end{minipage}%
\vspace{3mm}
\caption{Comparison of MAML and SimCLR features for Typiclust.}
\label{app:tbl:self_sup_typiclust}
\end{table*}

\begin{table*}[t!]
\begin{minipage}{\textwidth}
\centering
\fontsize{6}{9}\selectfont
\begin{tabu}{c|c|G|G|G|G|G|G}
\hline
\multirow{2}{*}{Dataset} & \multirow{2}{*}{Features} & \multicolumn{3}{c|}{1-Shot} & \multicolumn{3}{c}{5-Shot}  \\
\cline{3-8}
& & Fully strat. & Train strat. & Unstrat. & Fully strat. & Train strat. & Unstrat. \\
\hline
\hline
\multirow{2}{*}{Mini.} & MAML & \bftab{52.81} $\pm$ \bftab{1.16} &  21.91 $\pm$  0.24 & \bftab{36.21} $\pm$ \bftab{0.18} & \bftab{64.70} $\pm$ \bftab{0.91} & \bftab{42.07} $\pm$ \bftab{0.49} & 23.40 $\pm$ 0.36 \\
& SimCLR & 47.57 $\pm$ 0.42 & \bftab{25.35} $\pm$ \bftab{0.38} & 32.19 $\pm$ 0.43 & 64.33 $\pm$ 0.39 & 36.64 $\pm$ 0.58 & \bftab{26.16} $\pm$ \bftab{0.43} \\
\hline 
\multirow{2}{*}{FC100} & ProtoNet & \bftab{48.66} $\pm$ \bftab{0.16} & \bftab{32.86} $\pm$ \bftab{0.22} & \bftab{33.58} $\pm$ \bftab{0.19} & \bftab{51.11} $\pm$ \bftab{0.17} & \bftab{44.20} $\pm$ \bftab{0.24} & 44.40 $\pm$ 0.24 \\
& SimCLR & 31.40 $\pm$ 0.42 & 29.53 $\pm$ 0.42 & 28.39 $\pm$ 0.43 & 47.11 $\pm$ 0.39 & 39.33 $\pm$ 0.54 & \bftab{45.40} $\pm$ \bftab{0.52} \\
\Xhline{2\arrayrulewidth}
\end{tabu}
\end{minipage}%
\vspace{3mm}
\caption{Comparison of MAML and SimCLR features for ProbCover.}
\label{app:tbl:self_sup_probcover}
\end{table*}

\section{Sequential Active-Meta Learning} \label{app:sec:sequential}

Although iterative sampling is more common in active learning, we have focused on sampling a context set at once because of the following two reasons.

First, even though we iterative label additional samples, the features do not change in most of meta-learning algorithms except for MAML. 
Even for other optimization-based methods such as ANIL, since the feature extractor is not updated during adaptation on a context set, the features will stay the same for iterative process of active learning. 
As we demonstrated with ProtoNet in \cref{app:fig:seq} (c)-(d) (details about experiments are below), although we iteratively add more labeled samples, the performance does not change much as the features do not change. In this case, selecting $N \times K$ samples at once is not different from iterative process while it is cheaper.

Furthermore, if we iteratively add labeled samples, it will quickly go beyond few-shot regime in meta-learning, which is often not that practical in real world settings. 
Suppose we have a meta learner trained in $5$-way $1$-Shot. 
It is reasonable to add $5$ samples per iteration since it is the minimum number to cover all the classes. But only after $5$ iterations, it will go few-shot regime where we typically have $25$ labeled context samples. It is even less practical for $5$-Shot case.

\begin{figure*}[t!]
    \centering
    \begin{subfigure}[b]{0.24\textwidth} %
        \includegraphics[width=\textwidth]{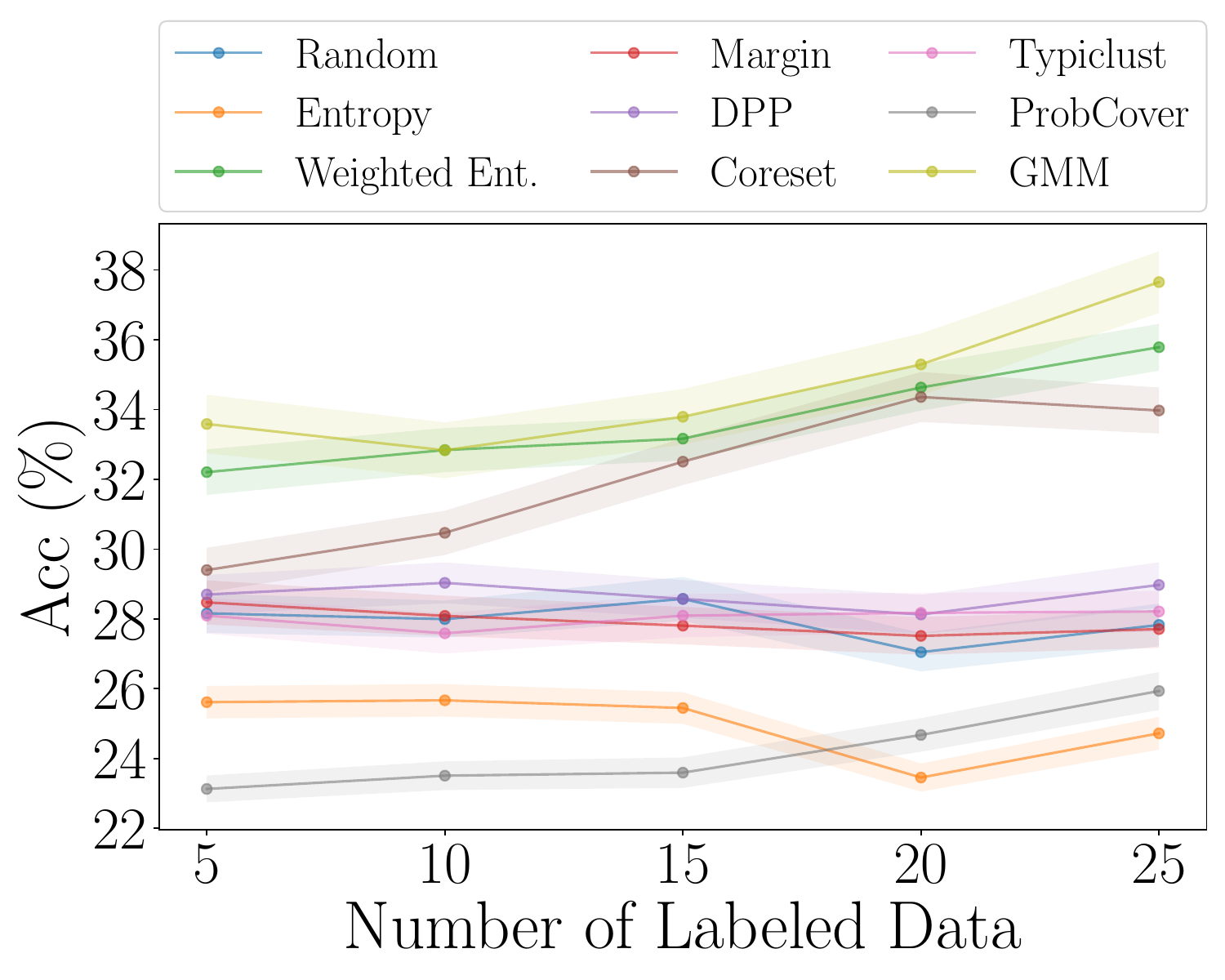}
        \caption{MAML: Train strat.}
        \label{app:fig:maml_mini_trainstrat}
    \end{subfigure}
    \hfill
    \begin{subfigure}[b]{0.24\textwidth} %
        \includegraphics[width=\textwidth]{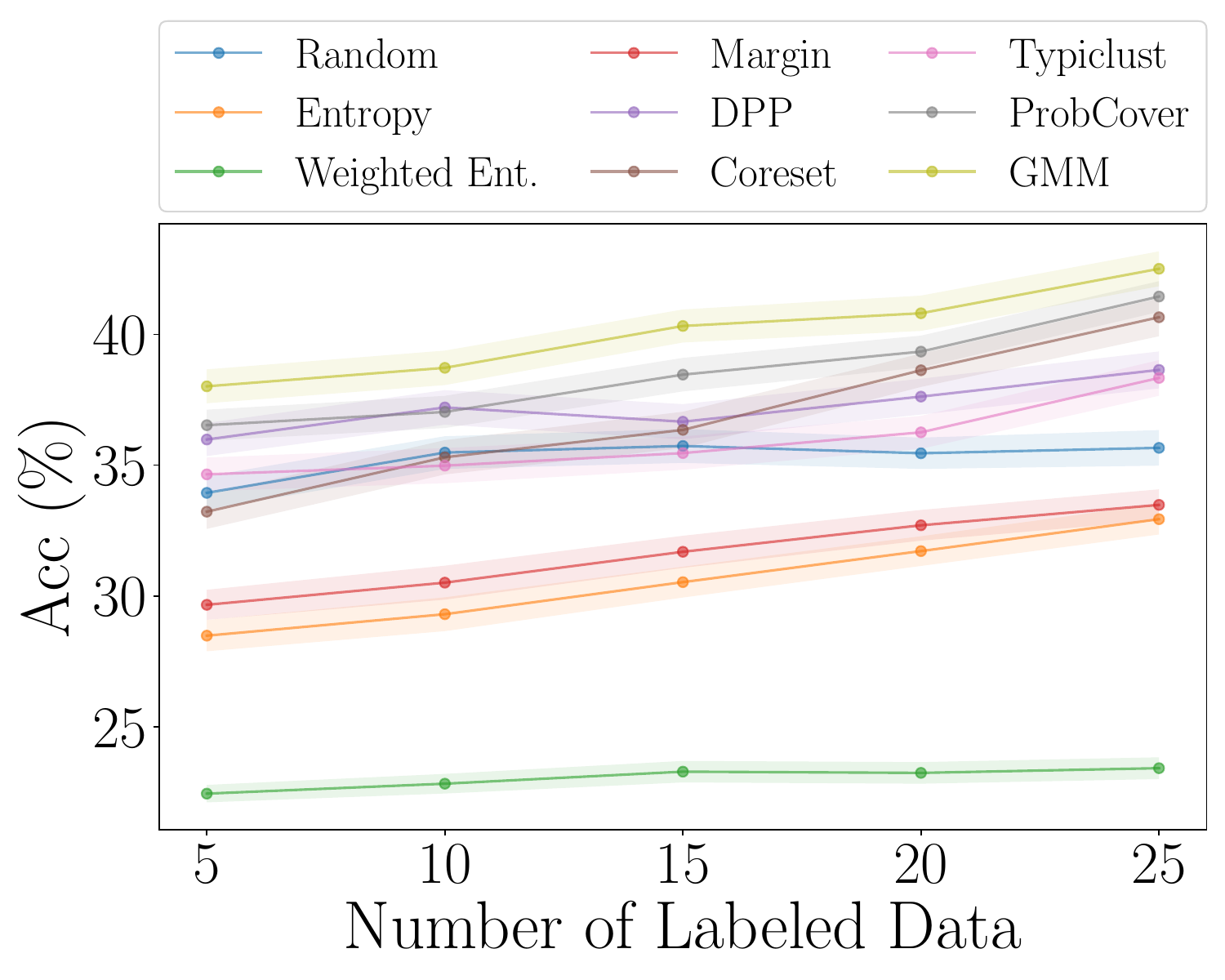}
        \caption{MAML: Unstrat.}
        \label{app:fig:maml_mini_unstrat}
    \end{subfigure}
    \hfill
    \begin{subfigure}[b]{0.24\textwidth} %
        \includegraphics[width=\textwidth]{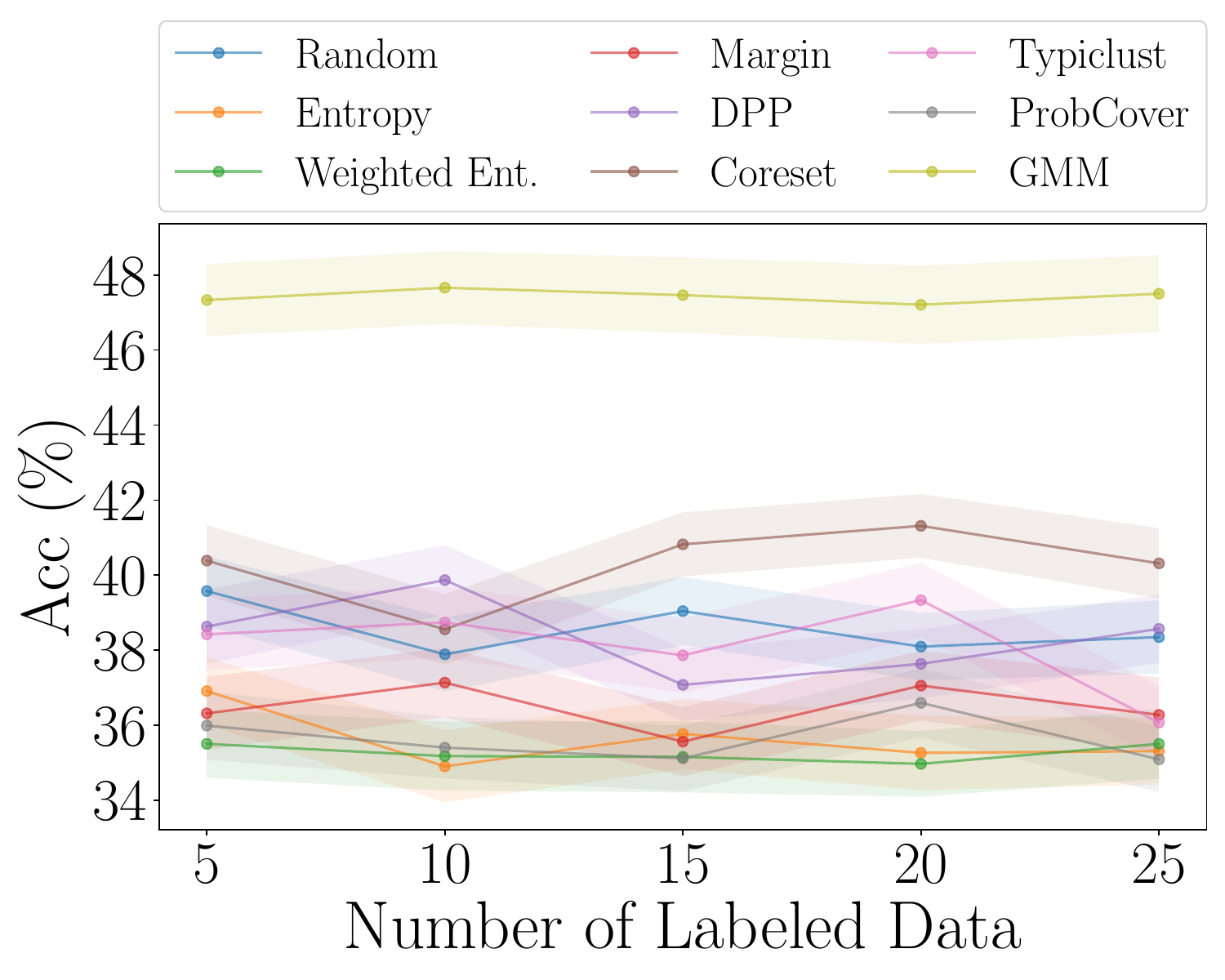}
        \caption{Proto: Train strat.}
        \label{app:fig:proto_mini_trainstrat}
    \end{subfigure}
    \hfill
    \begin{subfigure}[b]{0.24\textwidth} %
        \includegraphics[width=\textwidth]{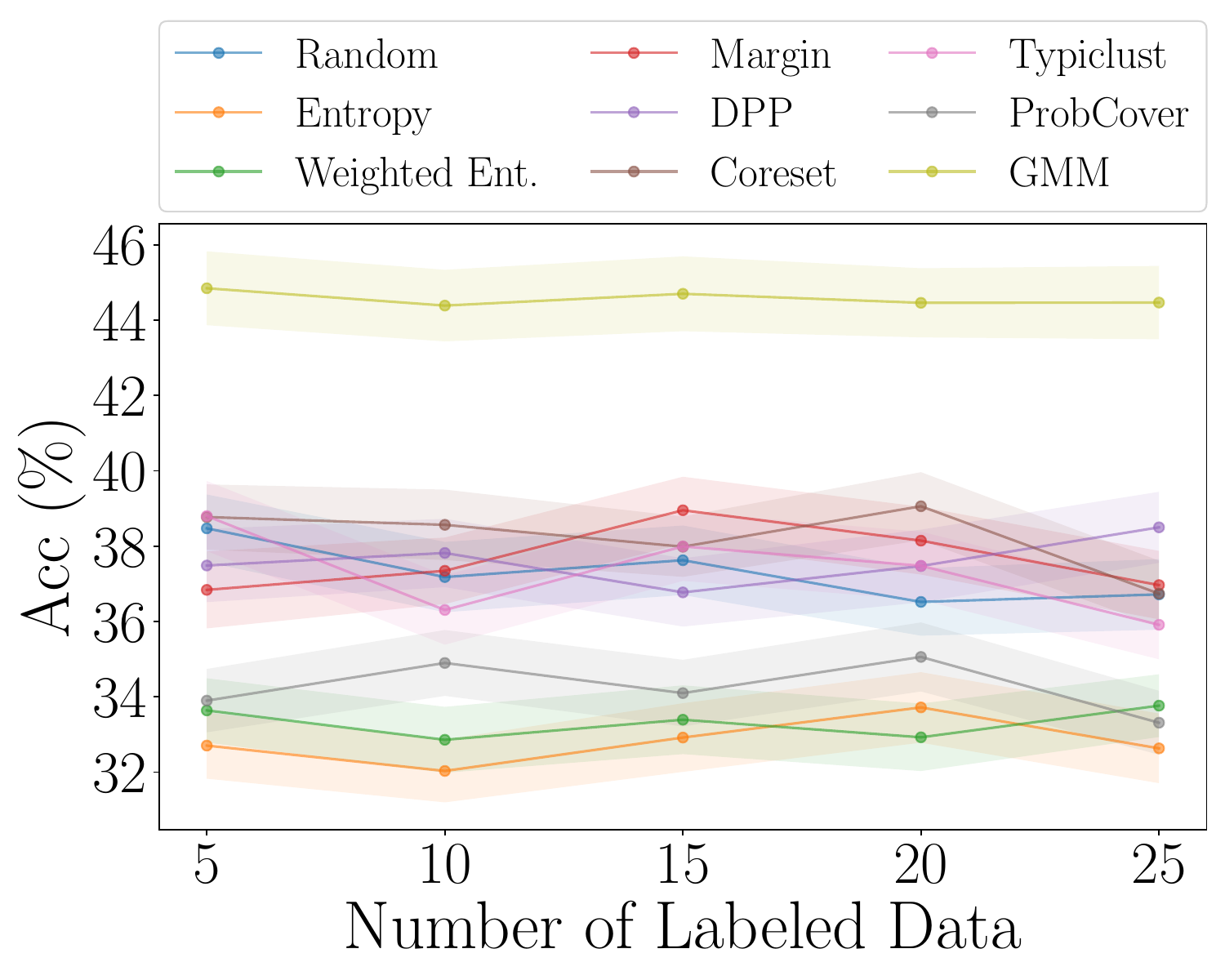}
        \caption{Proto: Unstrat.}
        \label{app:fig:proto_mini_unstrat}
    \end{subfigure}
    \caption{Test performance of MAML and ProtoNet on MiniImageNet with sequentially actively selected context sets. $5$ context samples are selected at each iteration until it reaches $25$.}
    \label{app:fig:seq}
    \vspace{3mm}
\end{figure*}

\begin{figure*}[t!]
    \centering
    \begin{subfigure}[b]{0.325\textwidth} %
        \includegraphics[width=0.99\textwidth]{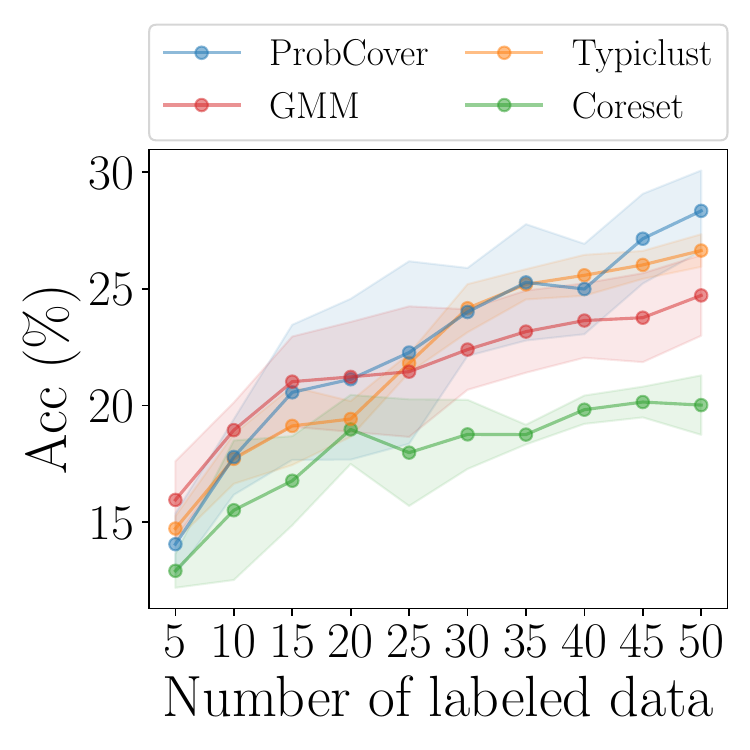}
        \caption{CIFAR10}
        \label{fig:img_cls_cifar10}
    \end{subfigure}
    \hfill
    \begin{subfigure}[b]{0.325\textwidth} %
        \includegraphics[width=0.99\textwidth]{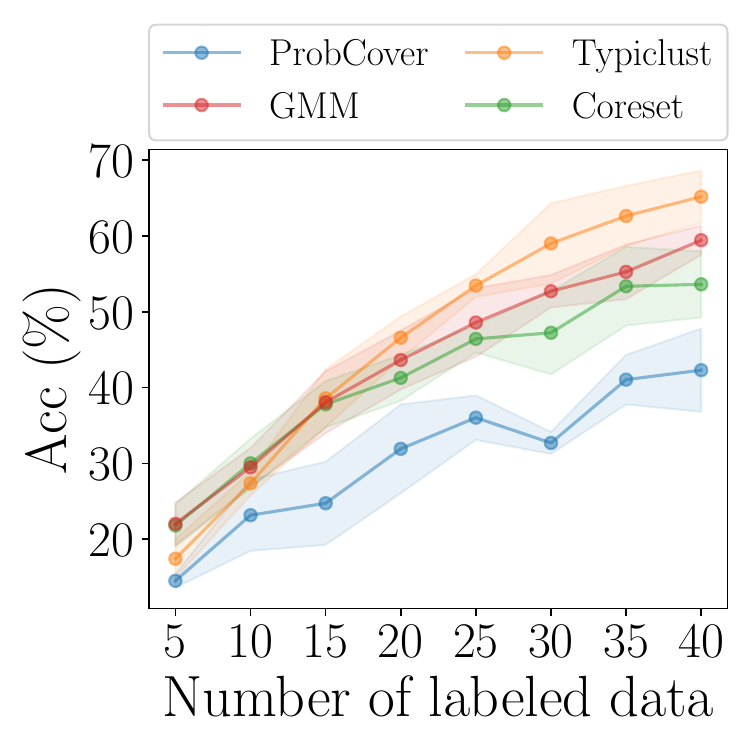}
        \caption{MNIST}
        \label{fig:img_cls_mnist}
    \end{subfigure}
    \hfill
    \begin{subfigure}[b]{0.325\textwidth} %
        \includegraphics[width=0.99\textwidth]{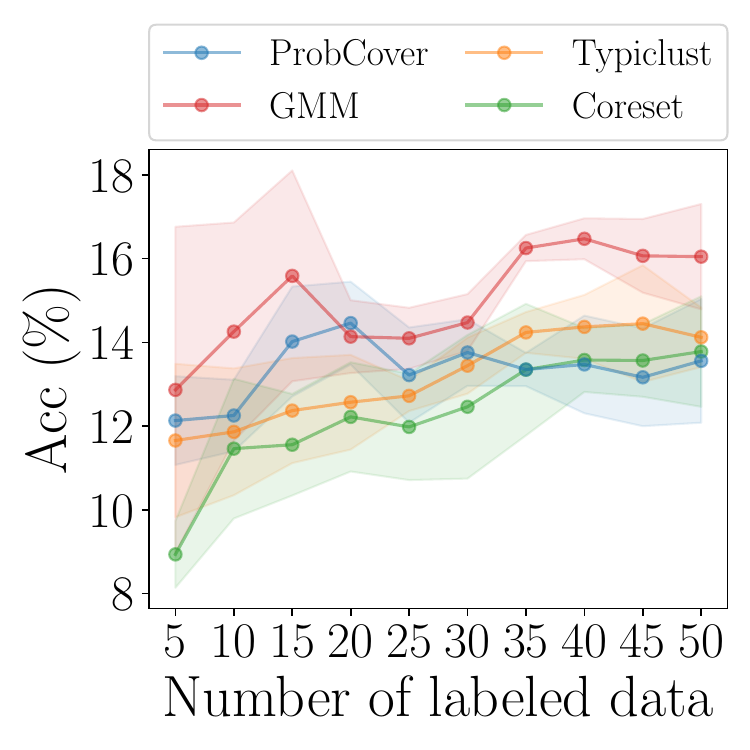}
        \caption{SVHN}
        \label{fig:img_cls_svhn}
    \end{subfigure}
    \caption{Low-budget active learning methods on image classification with very low budget.
    Mean and standard error of accuracy for three sets of SimCLR features, three runs per features.
    }
    \label{fig:img_cls}
    \vspace{3mm}
\end{figure*}

\cref{app:fig:seq} compare active learning methods for sequential setting where we select $5$ context samples at a time until the budget reaches $25$ samples.
Every time we select new context samples we may utilize them to maximize new label information.
For MAML, we update all the model parameters through adaptation steps.
It is, however, not applicable to the other meta-learning methods we use in this work including ProtoNet, since none of the other methods including optimization-based methods such as ANIL, do not update the parameters up to the penultimate layer.

As expected, the test performance of ProtoNet does not change much regardless of active learning methods.
But, the test performance of MAML gradually increases as we add more context samples.
In sequential active-meta learning, GMM still significantly outperforms other active learning methods.

\section{Fitting GMM using Expectation Maximization}
\label{app:sec:gmm}

In this section, we provide details about fitting GMM using the expectation maximization (EM) algorithm.
Although it is available in many literature, we add it here for completeness of our method.
The log-likelihood objective for a GMM is given by,
\begin{align}
    \ell(\theta) = \sum_{i=1}^{N} \log \left( \sum_{k=1}^{K} \pi_k \mathcal{N}(x_i | \mu_k, \Sigma_k) \right),
    \label{app:eq:gmm_obj}
\end{align}
where model parameters $\theta = \{(\pi_k, \mu_k, \Sigma_k)\}_{k=1}^K$ with $N$and $K$ being the number of samples and mixture components, respectively.

EM algorithm is an iterative algorithm where we alternatively conduct E-step and M-step as follows,
\begin{itemize}
    \item E-step: we compute the posterior probability $\omega_{ik}$ that represents $i$-th data point belongs to the $k$-th Gaussian component as,
    \begin{align}
        w_{ik} = \frac{\pi_k \mathcal{N}(x_i | \mu_k, \Sigma_k)}{\sum_{j=1}^{K} \pi_j \mathcal{N}(x_i | \mu_j, \Sigma_j)}
        \label{app:eq:e_step}
    \end{align}
    \item M-step: we maximize the log-likelihood in terms of the model parameters. Fortunately, for GMM, there are closed form solutions for each parameter.
    \begin{align}
        \pi_k &= \frac{1}{N} \sum_{i=1}^{N} w_{ik},\hspace{3mm}
        \mu_k = \frac{\sum_{i=1}^{N} w_{ik} x_i}{\sum_{i=1}^{N} w_{ik}},\hspace{3mm}
        \Sigma_k = \frac{\sum_{i=1}^{N} w_{ik} (x_i - \mu_k)(x_i - \mu_k)^T}{\sum_{i=1}^{N} w_{ik}}
        \label{app:eq:m_step}
    \end{align}
\end{itemize}
We repeat the E-step and M-step until convergence of the log-likelihood or for a fixed number of iteration time.
Please note that we use diagonal covariance $\Sigma_k$ since it is computationally efficient and often fits better in term of log-likelihood.

\section{Analysis on Low Budget Active Learning Methods}
\label{app:sec:further_low_budget}

We briefly explain why each active learning method does not perform as well as the proposed GMM method or even Random selection.
In this section, we discuss further on the inferiority of low budget active learning methods compared to GMM.
We conjecture it may be attributed to its implicit exploration of locally dense regions or inappropriate measure of representativeness.
Here we analyzed potential reasons of their failure in very low budget regime.
\begin{itemize}
    \item Typiclust: after conducting $k$-means, it selects samples for each cluster $j$, based on $$ \argmax_{x \in \text{clust}_j} \left( \frac{1}{K} \sum_{x_i \in KNN(x)} || x - x_i ||_2 \right) ^{-1} $$ where KNN denotes $k$-nearest neighbors of which size is fixed to 20. This measure seeks for locally dense region by selecting samples that are close to its nearest neighbors.
    \item ProbCover: it greedily finds the maximally covering samples given a fixed radius. This greedy algorithm provide $(1 - \frac{1}{e})$-approximation for the optimal solution but the gap with the optimal solution can be quite large. Also, the selection of the radius is hard as we discussed in \cref{app:sec:prob_cover}. When the radius is small, it tries to find samples that are in locally dense regions.
    \item DPP: it finds samples of which a kernel matrix (with a pre-defined kernel function) has the maximum determinant, which implicitly finds diverse samples. The determinant of a matrix, however, may not align with selecting maximum covering (or representative) samples. In particular, maximizing the determinant of the kernel matrix may lead to selecting samples far away from other samples.
\end{itemize}
Compared to these methods, GMM tries to find globally representative samples in non-greedy fashion (using expectation maximization). Also, its measure of covering other samples is in Mahalanobis distance, which intuitively makes more sense than the determinant of a kernel matrix as a measure. The proposed GMM method is also theoretically motivated by the \cref{prop:max-margin}, which says a classifier trained with the selected samples from GMM (cluster means) is a Bayes-optimal classifier under certain conditions.

\section{Training-Time Active Learning}
\label{app:sec:train_time}

\begin{figure}[h!]
\vspace{-7mm}
\begin{minipage}[t]{0.6\textwidth}
As mentioned in \cref{sec:conclusion}, we observe that active learning methods do not significantly change the generalization performance of meta learners when applied in meta-train time, which aligns with observations from \cite{aug_meta2021ni,support2020setlur}.
To empirically demonstrate it, we apply several active learning methods without stratification in the meta-train time for ProtoNet on MiniImageNet.
\cref{app:fig:training_time} shows among Random, DPP and GMM selections, one is not significantly better than another although the Entropy selection is significantly worse than them.
\end{minipage}
\begin{minipage}[t]{0.37\textwidth}
\centering
\includegraphics[width=1.0\textwidth,valign=t]{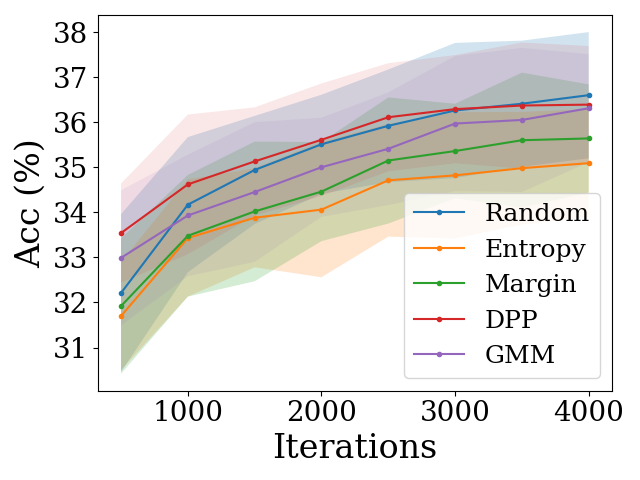}
\vspace*{-4.5mm}
\caption{Comparison of $\Pick_\theta^\train$.}
\label{app:fig:training_time}
\end{minipage}%
\end{figure}

\end{document}

%% file: math_commands.tex
\usepackage{amsmath,amsfonts,bm}

\def\1{\bm{1}}
\DeclareMathAlphabet{\mathsfit}{\encodingdefault}{\sfdefault}{m}{sl}
\SetMathAlphabet{\mathsfit}{bold}{\encodingdefault}{\sfdefault}{bx}{n}

\newcommand{\E}{\mathbb{E}}

\newcommand{\R}{\mathbb{R}}

\DeclareMathOperator*{\argmax}{arg\,max}
\DeclareMathOperator*{\argmin}{arg\,min}